\documentclass[twoside,11pt]{article}

\usepackage{arxiv}
\usepackage{amsmath}
\usepackage{graphics,graphicx}
\usepackage{hyperref}
\usepackage{xcolor}
\usepackage{subcaption}

\usepackage{tikz,array}
\usetikzlibrary{shapes.geometric}
\usetikzlibrary{shapes.arrows}
\usetikzlibrary{calc}
\usetikzlibrary{intersections}

\usepackage[utf8]{inputenc}

\newcommand{\grad}{{\mathrm{grad}}}

\newcommand{\stiefel}{{\mathrm{St}(p, n)}}
\newcommand{\sk}{{\mathrm{skew}}}
\newcommand{\sym}{{\mathrm{sym}}}
\newcommand{\tangent}{{\mathrm{T}_X\stiefel}}

\newcommand{\bR}{\mathbb{R}}    %
\newcommand{\Jacob}[2]{\mathcal{J}_{#2}\left(#1\right)}
\newcommand{\vect}[1]{\mathrm{vec}\left(#1\right)}
\newcommand{\cN}{\mathcal{N}} %
\newcommand{\cL}{\mathcal{L}} %
\newcommand{\vectHess}[1]{H_{#1}} %
\newcommand{\nfX}{\nabla f(X)} %
\newcommand{\inner}[2]{\left\langle#1,\,#2\right\rangle} %

\newcommand{\Lc}{L} %
\newcommand{\gradc}{L'} %
\newcommand{\revision}[1]{{#1}}
\newcommand{\finalrevision}[1]{{#1}}

\jmlrheading{1}{2023}{1-48}{2023-04-11}{10/00}{ablin22a}{Ablin, Vary, Gao, and Absil}

\ShortHeadings{Infeasible algorithms for orthogonality constraints}{Ablin, Vary, Gao, and Absil}
\firstpageno{1}

\begin{document}

\title{Infeasible Deterministic, Stochastic, and Variance-Reduction Algorithms for Optimization under Orthogonality Constraints}

\author{\name Pierre Ablin \email pierre.ablin@apple.com \\
       \addr Apple
       \AND
       \name Simon Vary \email simon.vary@uclouvain.be\\
       \addr ICTEAM/INMA, UCLouvain
       \AND
       \name Bin Gao \email gaobin@lsec.cc.ac.cn \\
       \addr ICMSEC/LSEC, AMSS, Chinese Academy of Sciences
       \AND
       \name P.-A. Absil \email pa.absil@uclouvain.be \\
       \addr ICTEAM/INMA, UCLouvain
       }

\editor{}

\maketitle

\begin{abstract}%
Orthogonality constraints naturally appear in many machine learning problems, from principal component analysis to robust neural network training. They are usually solved using Riemannian optimization algorithms, which minimize the objective function while enforcing the constraint.
However, enforcing the orthogonality constraint can be the most time-consuming operation in such algorithms. 
Recently, \citet{ablin2022fast} proposed the landing algorithm, a method with cheap iterations that does not enforce the orthogonality constraints but is attracted towards the manifold in a smooth manner.
This article provides new practical and theoretical developments for the landing algorithm.
First, the method is extended to the Stiefel manifold, the set of rectangular orthogonal matrices. We also consider stochastic and variance reduction algorithms when the cost function is an average of many functions.
We demonstrate that all these methods have the same rate of convergence as their Riemannian counterparts that exactly enforce the constraint, \revision{and converge to the manifold}.
Finally, our experiments demonstrate the promise of our approach to an array of machine-learning problems that involve orthogonality constraints.
\end{abstract}


\section{Introduction}

Letting $f$ a function from $\mathbb{R}^{n\times p}$ to $\mathbb{R}$, we consider the problem of optimizing $f$ with orthogonality constraints:
\begin{equation}
    \min_{X\in\bR^{n\times p}} f(X), \qquad \text{s.t.}\quad X \in \stiefel = \{X\in\mathbb{R}^{n\times p}|\enspace X^\top X = I_p \}, \label{eq:optim_stiefel}
\end{equation}
where $\stiefel$ is \emph{the Stiefel manifold}.
Such a problem appears naturally in many machine learning problems as a way to control dissimilarity between learned features, e.g.~in principal component analysis (PCA)~\citep{hotelling1933analysis}, independent component analysis (ICA)~\citep{hyvarinen1999fast,TheCasAbs2009,ablin2018faster}, canonical correlation analysis~\citep{hotelling1936relations}, and more recently, for the training of neural networks to improve their stability \citep{arjovsky2015unitary, zhang2021orthogonality, wang2020orthogonal} and robustness against adversarial attacks \citep{cisse2017parseval, li2019orthogonal, li2019preventing, singla2021skew}. %

Riemannian optimization techniques are based on the observation that the orthogonality constraints in \eqref{eq:optim_stiefel} define a smooth matrix manifold \citep{absil2008optimization, boumal2023introduction} called the Stiefel manifold. The geometry of the manifold constraint allows for the extension of optimization techniques from the Euclidean to the manifold setting, including second-order methods~\citep{absil2007trust}, stochastic gradient descent~\citep{bonnabel2013stochastic}, and variance-reduced methods \citep{zhang2016riemannian, tripuraneni2018averaging, zhou2019faster}. 

A crucial part of Riemannian optimization methods is the use of \emph{retraction} \citep{ADM2002, absil2012projection}, which is a projection map on the manifold preserving the first-order information, and ensures that the iterates remain on the manifold. Computing retractions with orthogonality constraints involves linear algebra operations, such as matrix exponentiation, inverse, or QR decomposition. In some applications, e.g.,~when evaluating the gradient is relatively cheap, computing the retraction is the dominant cost of the optimization method. 
Unlike Euclidean optimization, ensuring that the iterates move on the manifold can be more costly than computing the gradient direction.

Additionally, the need to perform retractions, and more generally, to take the manifold's curvature into account, causes challenges in developing accelerated techniques in Riemannian optimization that the community has just started to overcome~\citep{, becigneul2018riemannian, alimisis2021momentum, criscitiello2022negative}. As a result, practitioners in deep learning sometimes rely on the use of adding a squared penalty term and minimize $f(X) + \lambda \cN(X)$ with $\cN(X) = \| X^\top X - I_p \|^2/4$ in the Frobenius norm,  which does not perfectly enforce the constraint. 

Unlike Riemannian techniques, where the constraint is exactly enforced in every iteration, and the squared penalty method, where the optimum of the problem is not exactly on the manifold, we employ a method that is in between the two. Motivated by the previous work of \cite{ablin2022fast} for the orthogonal matrix manifold, we design an algorithm that does not enforce the constraint exactly at every iteration but instead controls the distance to the constraint employing inexpensive matrix multiplication.
The iterates \emph{land} on the manifold exactly at convergence with the same convergence guarantees as standard Riemannian methods for solving~\eqref{eq:optim_stiefel}.

The following subsection provides a brief prior on the current optimization techniques for solving \eqref{eq:optim_stiefel}. The rest of the paper is organized as follows:
\begin{itemize}
  \item{In Section~\ref{sec:landing} we extend the landing algorithm to $\stiefel$. By bounding the step size, we guarantee that the iterates remain close to the manifold. We show that a function introduced in~\citep{gao2019parallelizable} based on the augmented Lagrangian is a merit function for the landing algorithm. \revision{This forms the cornerstone of our improved analysis over that of \cite{ablin2022fast}.}}
  \item{We use these tools for the theoretical analysis of the basic and variants of the landing algorithm in Section~\ref{sec:algorithms}. In Subsection~\ref{subsec:landing_gradient}, thanks to the
  new merit function, we greatly improve the theoretical results of \cite{ablin2022fast}: we demonstrate that the basic landing method with constant step size achieves $\inf_{k\leq K}\|\grad f(X_k)\|^2 =O(K^{-1})$. In Subsection~\ref{subsec:sgd}, we introduce a stochastic algorithm dubbed stochastic landing algorithm. We show that with a step size decaying in $K^{-\frac12}$, it achieves $\inf_{k\leq K}\|\grad f(X_k)\|^2 =O(K^{-\frac12})$. In Subsection~\ref{subsec:SAGA}, we extend the SAGA algorithm and show that the SAGA landing algorithm achieves $\inf_{k\leq K}\|\grad f(X_k)\|^2 =O(K^{-1})$. We recover each time the same convergence rate and sample complexity as the classical Riemannian counterpart of the methods---that uses retractions. \revision{We also demonstrate the convergence of all these methods to the Stiefel manifold.}}
  \item{In Section \ref{sec:experiments}, we numerically demonstrate the merits of the method when computing retractions is a bottleneck of classical Riemannian methods.}
\end{itemize}

Regarding the convergence speed results, the reader should be aware that we use the \emph{square norm} of the gradient as a criterion, while some readers might be more familiar with results stated without a square.

\revision{This paper improves on the results of \cite{ablin2022fast} in several important and non-trivial directions: i) the landing algorithm is extended to the Stiefel manifold, while \cite{ablin2022fast} only consider the square $n=p$ case, ii) stochastic and variance reduced extensions of the landing algorithm are proposed, while \cite{ablin2022fast} only consider gradient descent, and iii) we greatly improve the theoretical results, thanks to a new analysis. We essentially prove that all proposed algorithms converge at the same rate as their Riemannian counterparts using the same step-size schedule, and converge to the manifold. \cite{ablin2022fast} could only prove convergence of the base landing method using \emph{decaying} step sizes, leading to a sub-optimal $K^{-\frac13}$ convergence rate, while our result \autoref{prop:convergence_landing} proves that the base landing algorithm converges at a $K^{-1}$ rate with constant step-size, heavily improving the rate.}\footnote{When this work was ready to be made public, the preprint~\citep{schechtman2023orthogonal} appeared that pursues similar goals. It addresses the more general problem of equality-constrained optimization, it uses a different merit function than the one we introduce in Section~\ref{sec:merit} and it does not consider variance reduction as we do in Section~\ref{subsec:SAGA}. The numerical experiments also differ considerably, as they only consider a Procrustes problem, while we experiment with deep neural networks.}
\paragraph{Notation} The norm of a matrix $\|M\|$ denotes the Frobenius norm, i.e., the vectorized $\ell_2$ norm. We let $I_p$ denote the $p\times p$ identity matrix, and $\stiefel$ denote the Stiefel manifold, which is the set of $n\times p$ matrices such that $X^\top X = I_p$. The tangent space of the Stiefel manifold at $X\in \stiefel$ is denoted by $\tangent$. The projection on the set of $p \times p$ skew-symmetric matrices denoted by $\sk_p\subset \bR^{p\times p}$ is $\sk(M) = \frac12(M - M^\top)$ and on the set of symmetric matrices is $\sym(M) = \frac12(M + M^\top)$. The exponential of a matrix is $\exp(M)$. The gradient of a function $f:\mathbb{R}^{n\times p}\to \mathbb{R}$ is denoted by $\nabla f$ and we define its Riemannian gradient as $\grad f(X) = \sk(\nabla f(X)X^\top)X$ for all $X\in\mathbb{R}^{n\times p}$ as explained in detail in Section~\ref{sec:landing}. We say that a function is $\Lc$-smooth if it is differentiable and its gradient is $\Lc$-Lipschitz. The constant $\Lc$ is then the smoothness constant of $f$.

\subsection{Prior Work on Optimization with Orthogonality Constraints}

Equation~\eqref{eq:optim_stiefel}
is an optimization problem over a matrix manifold. In the literature, we find two main approaches to solving this problem, reviewed next.
\subsubsection{Trivializations} This approach proposed by \citet{lezcano2019cheap, lezcano2019trivializations} consists in finding a differentiable surjective function $\phi:E\to \stiefel$ where $E$ is a
Euclidean space, and to solve
\begin{equation}
  \label{eq:trivialization}
  \min_{A\in E}f(\phi(A))\enspace.
\end{equation}
The main advantage of this method is that it turns a Riemannian optimization problem on $\stiefel$ into an optimization on a Euclidean space,
for which we can apply all existing Euclidean methods, such as
gradient descent, stochastic gradient descent, or Adam. This is especially convenient in deep learning, where most standard optimizers are not meant to handle Riemannian constraints.
However, this method has two major drawbacks. First, it can drastically change the optimization landscape. Second, the gradient of the function that is optimized is, following the chain rule, $\nabla (f \circ \phi) = \left(\frac{\partial \phi}{\partial z}\right)^\top\nabla f \circ \phi$, and the Jacobian-vector product can be very expensive to compute.

To give a concrete example, we consider the parametrization $\phi$ used by \citet{singla2021skew}: $\phi(A) = \exp(A)\begin{pmatrix}
  I_p \\ 0
\end{pmatrix}$, where $A\in\sk_n$, with $\sk_n$ the set of $n\times n$ skew symmetric matrices. We see that computing this trivialization requires computing the exponential of a potentially large $n\times n$ matrix. Furthermore, when computing the gradient of $f\circ \phi$, one needs to compute the Jacobian-vector product with a matrix exponential, which requires computing a larger $2n \times 2n$ matrix exponential~\citep{lezcano2019cheap}: this becomes prohibitively costly when $n$ is large.
\subsubsection{Riemannian optimization}
\label{subsec:riem_optim}
This approach consists in extending the classical Euclidean methods such as
gradient descent or stochastic gradient descent to
the Riemannian setting. For instance, consider Euclidean gradient descent to minimize $f$ in the Euclidean setting, which iterates $X_{k+1}= X_k - \eta \nabla f(X_k)$ where $\eta > 0$ is a step size. There are two ingredients to transform it into a Riemannian method. First, the Euclidean gradient $\nabla f(X)$ is replaced by the Riemannian gradient $\grad f(X)$. Second, the subtraction is replaced by a retraction $\mathcal{R}$, which allows moving while staying on the manifold. We obtain the iterations of the Riemannian gradient descent:
\begin{equation}
  \label{eq:riemannian_gradient_descent}
  X_{k+1} = \mathcal{R}(X_k, -\eta \grad f(X_k))\enspace.
\end{equation}
In the case of $\stiefel$~\citep{edelman1998geometry}, the tangent space at $X$ is the set
\begin{equation}
  \label{eq:tangent_space}
  \tangent = \{X\Omega + X_{\perp}K:\Omega \in\sk_p, K\in\mathbb{R}^{n - p \times p}\} = \{WX: W \in \sk_n\}
\end{equation}
and the Riemannian gradient with respect to the canonical metric~\citep{edelman1998geometry} is 
\begin{equation}
  \label{eq:riemannian_gradient}
  \grad f(X) = \sk(\nabla f(X)X^\top)X
\end{equation}
where $\sk(M) = \frac12(M - M^\top)$ is the skew-symmetric part of a matrix. In~\eqref{eq:riemannian_gradient}, for convenience, we have omitted a factor of $2$; compare with~\cite[\S 3]{gao2022optimization}. This omission is equivalent to magnifying the canonical metric by a factor of $2$.
From a computational point of view,
computing this gradient is cheap: it can be rewritten, for $X\in\stiefel$, as $\grad f(X) = \frac12(\nabla f(X) - X\nabla f(X)^\top X)$, which can be computed with one matrix-matrix product of size $p\times n$ and $n\times p$, and one matrix-matrix product of size $n \times p$ and $p\times p$.

A retraction $\mathcal{R}(X, Z) = Y$ is a mapping that takes as inputs $X\in\stiefel$ and $Z\in\tangent$, and outputs $Y\in \stiefel$, such that it ``goes in the direction of $Z$ at the first order'', i.e.,
we have $Y = X + Z + o(\|Z\|)$. It defines a way to move on the manifold $\stiefel$. We give several examples, where we write $Z =X\Omega + X_{\perp}K = WX$ in view of~\eqref{eq:tangent_space}.
\begin{enumerate}
  \item \emph{Exponential retraction}: \begin{equation}
    \label{eq:exp_retraction}
    \mathcal{R}(X, Z) =   \begin{pmatrix}
    X & X_\perp
  \end{pmatrix}\exp\begin{pmatrix}
    \Omega & -K^\top \\
    K & 0
  \end{pmatrix}
  \begin{pmatrix}
    I_p  \\
    0
  \end{pmatrix}\enspace .
\end{equation}
  This is the exponential map---that follows geodesics---for the canonical metric on the manifold~\citep[(10)]{ZimmermannHueper2022}. The most expensive computations are a matrix exponentiation of a matrix of size $2p\times 2p$ and a matrix-matrix product between matrices of size $n\times 2p$ and $2p\times 2p$.
  \item \emph{Cayley retraction}:
  \begin{equation*}
      \label{eq:cayley_retraction}
      \mathcal{R}(X, Z) = (I_n - \frac{W}{2})^{-1}(I_n + \frac{W}{2}) X  \enspace.
  \end{equation*}
  Though the inverse exists for any $W\in\sk_n$, it requires solving a $n\times n$ linear system that dominates the cost. When $Z$ takes the form~\eqref{eq:riemannian_gradient} and $\nabla f(X)$ is available, this retraction can be computed in $8np^2 + O(p^3)$ flops, see~\citet[\S 2.2]{WenYin2012}
  \item \emph{Projection retraction}:\begin{equation}
    \label{eq:proj_retraction}
    \mathcal{R}(X, Z) = \mathrm{Proj}(X + Z), \enspace \text{with }\mathrm{Proj}(Y) = Y(Y^\top Y)^{-\frac12}\enspace.
\end{equation}
Computing this retraction requires computing the inverse-square root of a matrix, which is also a costly linear algebra operation.
\end{enumerate}

These operations allow us to implement Riemannian gradient descent. Many variants have then been proposed to accelerate convergence, among which trust-region
methods~\citep{absil2007trust} and Nesterov acceleration~\citep{ahn2020nesterov}.

In most machine learning applications, the function $f$ corresponds to empirical risk minimization, and so it has a sum structure. It can be written as:
\begin{equation}
  \label{eq:erm}
  f(X) = \frac1N\sum_{i=1}^Nf_i(X)\enspace,
\end{equation}
where $N$ is the number of samples and each $f_i$ is a ``simple'' function. In the Euclidean case, stochastic gradient descent
(SGD)~\citep{robbins1951stochastic} is the algorithm of choice to minimize such a function. At each iteration, it takes a random function $f_i$ and goes in the direction opposite to its gradient. Similarly, in the Riemannian case, we can implement Riemannian-SGD~\citep{bonnabel2013stochastic} by iterating
\begin{equation}
  \label{eq:sgd}
  X_{k+1} = \mathcal{R}(X_k, -\eta_k\grad f_i(X_k)),\enspace \text{where } i \sim \mathcal{U}[1, N] 
\end{equation}
with $i \sim \mathcal{U}[1, N] $ meaning that index $i$ is drawn from the discrete uniform distribution between $1$ and $N$ at each iteration. 
This method only requires one sample at each iteration; hence it scales gracefully with $N$. However, its convergence is quite slow and typically requires diminishing step sizes.

Variance reduction techniques~\citep{johnson2013accelerating, schmidt2017minimizing, defazio2014saga} are classical ways to mitigate this problem and allow to obtain algorithms that are stochastic (i.e.,
use only one sample at a time) and
that provably converge while having a constant step size, with a faster rate of convergence than SGD.

\finalrevision{
    \subsection{Infeasible methods}
    Like the present work, several infeasible methods have been proposed to solve the optimization problem, while not enforcing the orthogonality constraint at each iteration. 
    \citet{ablin2018faster} propose a method we extend in this work in several directions. 
    \citet{gao2019parallelizable} propose an approximated augmented Lagrangian method, taking steps in directions that approximate the gradient of the augmented Lagrangian. 
    We draw inspiration from this work in order to perform our theoretical analysis; a detailed account of the differences with our algorithm is given in \autoref{sec:merit}.
    Finally, \citet{liu2024penalty} propose a proximal method to solve possibly non-smooth problems on the manifold, with an approximate orthogonalization scheme at each iteration. 
    Owing to the more complicated problem structure they consider, the proposed method invokes at each step an inner iterative solver to compute a proximal operator.
}

\section{Geometry of the Landing Field and its Merit Function}
\label{sec:landing}
For $\lambda > 0$ and $X\in\mathbb{R}^{n\times p}$, we define the landing field as 
\begin{equation}
  \label{eq:landing_field}
  \Lambda(X) =\grad f(X) + \lambda X(X^\top X - I_p),
\end{equation}
where $\grad f(X) = \sk(\nabla f(X)X^\top)X$.
This field will be used to define iterates $\tilde{X} = X - \eta \Lambda(X)$ and different variants in the next sections.
We define 
\begin{equation}
\label{eq:norm_function}
    \mathcal{N}(X) = \frac14\|X^\top X - I_p\|^2
\end{equation}
where $\|\cdot\|$ is the Frobenius norm so that the second term in~\eqref{eq:landing_field} 
is $\lambda \nabla \mathcal{N}(X)$. 

Here, $\grad f(X)$ denotes the extension to all $X\in\mathbb{R}^{n\times p}$ of formula~\eqref{eq:riemannian_gradient}. It thus coincides with the Riemannian gradient when $X\in\stiefel$. This extension has several attractive properties.
First, for all full-rank $X\in\mathbb{R}^{n\times p}$, $\grad f(X)$ can still be seen as a Riemannian gradient of $f$ on the manifold $\{Y\in\mathbb{R}^{n\times p} \mid Y^\top Y = X^\top X\}$ with respect to a canonical-type metric, as shown in~\cite[Proposition~4]{gao2022optimization}.
Second, this formula yields orthogonality between the two terms of the field $\Lambda$, for any $X$. Indeed, we have $\langle\grad f(X), X(X^\top X - I_p)\rangle = \langle\sk(\nabla f(X)X^\top),X(X^\top X - I_p)X^\top\rangle$, which cancels since it is the scalar product between a skew-symmetric and a symmetric matrix.

The intuition behind the landing direction is fairly simple: the component $-\grad f(X)$ is here to optimize the function, while the term $-X(X^\top X - I_p)$ attracts $X$
towards $\stiefel$.
More formally, since these two terms are orthogonal, the field cancels if and only if both terms cancel. The fact that $X(X^\top X - I_p)=0$ gives, assuming $X$ injective,  $X^\top X = I_p$, hence $X\in \stiefel$. Then, the fact that $\grad f(X) =0$ combined with $X\in\stiefel$ shows that $X$ is a first-order critical point of the function $f$ on the manifold.
This reasoning is qualitative: the next part formalizes this geometrical intuition.

\subsection{Geometrical Results and Intuitions}

In the remainder of the paper, we will always consider algorithms whose iterates stay close to the manifold $\stiefel$. 
We measure this closeness with the function $\mathcal{N}(X)$ (introduced in~\eqref{eq:norm_function}), and define the \emph{safe region} as 
\begin{equation}
    \label{eq:safe_region}
    \stiefel^{\varepsilon} = \{X\in\mathbb{R}^{n\times p}|\enspace \mathcal{N}(X)\leq \frac14 \varepsilon^2\} = \{X\in\mathbb{R}^{n\times p}|\enspace \|X^\top X - I_p\|\leq \varepsilon\} 
\end{equation}
for some $\varepsilon$ between $0$ and $1$. A critical part of our work is to ensure that the iterates of our algorithms remain in $\stiefel^\varepsilon$, which in turn guarantees the following bounds on the singular values of $X$.
All proofs are deferred to the appendix.
\begin{lemma}
\label{lemma:singular_values}
    For all
    $X\in \stiefel^\varepsilon$, the singular values of $X$ are between $\sqrt{ 1 - \varepsilon}$ and $ \sqrt{1 + \varepsilon}$.
\end{lemma}
Note that when $\varepsilon=0$, the singular values of $X$ are all equal to $1$, thus making the columns of the matrix orthogonal and ensuring that $X\in\stiefel$.

As noted before, a critical feature of the landing field is the orthogonality of the two components, which holds between $\nabla \cN(X)$ and any direction $AX$ with a skew-symmetric matrix $A$. In order for the results to generalize to the stochastic and variance reduction setting, in the rest of this section we consider a more general form of the landing field as
\begin{equation}
\label{eq:skew_field}
    F(X, A) = AX + \lambda \nabla \cN(X),
\end{equation}
where $A$ is an $n\times n $ skew-symmetric matrix. The formula of the landing field in \eqref{eq:landing_field} is recovered by taking $A = \sk(\nabla f(X)X^\top)$ in the above equation \eqref{eq:skew_field}, that is to say $\Lambda(X) = F(X, \sk(\nabla f(X)X^\top))$.
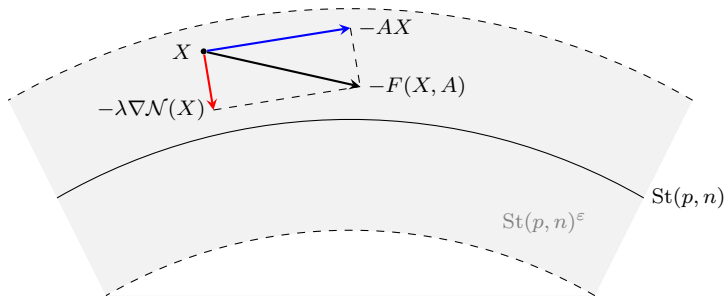
\begin{figure}[t]
	\scriptsize
	\centering
		{ 
	\begin{tikzpicture}[scale=.65]
		\fill[gray!10] (5,-2) -- (7,2) -- (7,2) arc (60:120:14) -- (-7,2) -- (-5,-2) -- (5,-2) arc (60:120:10);
		\fill[white] (5,-2) -- (5,-2) arc (60:120:10) -- (5,-2);
				
		\coordinate [fill=black,inner sep=.8pt,circle,label=180:{$X$}] (X) at (-3,3);
		\coordinate [label=0:{$-AX$}] (G) at (0,3.48);
		\coordinate [label=180:{$-\lambda\nabla\mathcal{N}(X)$}] (N) at (-2.8,1.8);
		\coordinate [label=0:{$-F(X,A)$}] (L) at ($(G)-(X)+(N)-(X)+(X)$);
				
		\coordinate [label=0:{$\mathrm{St}(p,n)$}] (Stiefel) at (6,0);
		\coordinate [label=180:{\color{gray}$\mathrm{St}(p,n)^{\varepsilon}$}] (eps) at (5,-0.5);
				
		\draw[-{stealth[black]},black,thick] (X) --  (L);
		\draw[-{stealth[blue]},blue,thick] (X) --  (G);
		\draw[-{stealth[red]},red,thick] (X) --  (N);
		\draw[dashed] (G) -- (L);
		\draw[dashed] (N) -- (L);
				
		\draw (6,0) arc (60:120:12);
		\draw[dashed] (7,2) arc (60:120:14);
		\draw[dashed] (5,-2) arc (60:120:10);
	\end{tikzpicture}
	}
	\caption{Illustration of the geometry of the field $F$. Note the orthogonality of the two components.\label{fig:diagram_landing}}
\end{figure}

Fields of the form~\eqref{eq:skew_field} will play a key role in our analysis, exhibiting interesting geometrical properties which stem from the orthogonality of the two terms: $AX$ and $\nabla\cN(X) = X(X^\top X - I_p)$ are orthogonal. Figure~\ref{fig:diagram_landing} illustrates the geometry of the problem and of the field $F$. We have the following inequalities:
\begin{proposition}
\label{prop:bound_landing_norm}
  For all $X\in\stiefel^\varepsilon$ and $A \in \sk_n$, the norm of the field~\eqref{eq:skew_field} admits the following bounds:
  $$
 \|AX\|^2 + 4\lambda^2(1-\varepsilon)\mathcal{N}(X) \leq \|F(X,A)\|^2\leq\|AX\|^2 + 4\lambda^2(1+\varepsilon)\mathcal{N}(X)
  $$
\end{proposition}

The orthogonality of the two terms also ensures that going in the direction of $-F(X, A)$ allows us to remain in the safe region $\stiefel^\varepsilon$ as long as the step size is small enough.

\begin{lemma}[Step-size safeguard] \label{lemma:safe_step}
 Let $X\in \stiefel^\varepsilon$, $A \in \sk_n$, and consider the update $\tilde{X} =X - \eta F(X, A)$, where $\eta > 0$ is a step size and $F(X,A)$ is the field~\eqref{eq:skew_field}. Define $g = \| F(X, A)\|$ and $d = \|X^\top X - I_p\|$. If the step size satisfies 
  \begin{equation} \label{eq:safe_step size}
      \eta \leq \eta(X) := \min \left\{ \frac{\lambda d (1-d) + \sqrt{\lambda^2 d^2 (1-d)^2 + g^2 (\varepsilon - d)}}{g^2}, \frac{1}{2\lambda} \right\},
  \end{equation}
  then the next iterate $\tilde{X}$ remains in $\stiefel^\varepsilon$.
  \end{lemma}

This lemma is of critical importance, both from a practical and theoretical point of view. In practice, at each iteration of our algorithms introduced later, we always compute the safeguard step~\eqref{eq:safe_step size}, and if the safeguard step is smaller than the prescribed step size, we use the safeguard step instead. 
Note that the formula for the safeguard step only involves quantities that are readily available when the field $F$ has been computed: computing $\eta(X)$ at each iteration does not add a significant computational overhead.

We can furthermore lower bound the safeguard step in~\eqref{eq:safe_step size} 
by a quantity independent of $X$:
\begin{lemma}[A lower-bound on the step-size safeguard]\label{lemma:safe_step size_lower}
    Assuming that $\|A X\|_F\leq \tilde{a}$, we have that the upper-bound in Lemma~\ref{lemma:safe_step} is lower-bounded by
    \begin{equation}
        \eta(X) \geq \eta^*(\tilde{a}, \varepsilon, \lambda) 
    \end{equation}
    where the quantity $ \eta^*(\tilde{a}, \varepsilon, \lambda)$, given in Appendix~\ref{app:proof:safe_step}, is positive for any $\lambda, \varepsilon$ and $\tilde{a}$.
\end{lemma}

This lemma serves to prove that the iterates of our algorithms all stay in the safe region provided that the step size is small enough. We have:
\begin{proposition}
\label{prop:recursive_step_size}
    Consider a sequence of iterates $X_k$ defined by recursion, starting from $X_0\in \stiefel^\varepsilon$. We assume that there is a family of maps $\mathcal{A}_k(X_0, \dots, X_k) = A_k\in\sk_n$ such that $X_{k+1} = X_k-\eta_kF(X_k, A_k)$ where $\eta_k > 0$ is a step size. In addition, we assume that there is a constant $\tilde{a}>0$ such that for all $X_0, \dots, X_k \in \stiefel^\varepsilon$, we have $\|\mathcal{A}_k(X_0, \dots, X_k)X_k\|\leq \tilde{a}$. Then, if all $\eta_k$ are such that $\eta_k\leq \eta^*(\tilde{a}, \varepsilon, \lambda)$ with $\eta^*$
    defined in Lemma~\ref{lemma:safe_step size_lower}, we have that all iterates satisfy 
    $X_k\in\stiefel^\varepsilon$.
\end{proposition}
This proposition shows that an algorithm that follows a direction of the form~\eqref{eq:skew_field} with sufficiently small steps will stay within the safe region $\stiefel^\varepsilon$.
The definition of the maps $\mathcal{A}_k$ is cumbersome as it depends
on the past iterates, but it is needed to handle the variance reduction algorithm that we study later. 
This result is central to this article since all algorithms considered in Section~\ref{sec:algorithms} produce sequences that satisfy the hypothesis of Proposition~\ref{prop:recursive_step_size}.

\subsection{A Merit Function}
\label{sec:merit}
The next proposition defines a smooth merit function for the landing 
field $\Lambda(X)$ defined in~\eqref{eq:landing_field}. The existence of such a merit function is central to a simple analysis of the landing algorithm and its different extensions. 
We consider as in~\citep{gao2019parallelizable}:
\begin{equation} \label{eq:merit_function}
    \mathcal{L}(X) = f(X) + h(X) + \mu \mathcal{N}(X),
\end{equation}
where 
$h(X) = -\frac12\langle \sym(X^\top \nabla f(X)), X^\top X - I_p\rangle$ and $\mu>0$, which is suitably chosen in the following result.

\begin{proposition}[Merit function bound] \label{prop:lyapunov}
    Let $\cL(X)$ be the merit function defined in \eqref{eq:merit_function}. For all $X \in \stiefel^\varepsilon$ we have with $\nu = \lambda \mu$:%
    \begin{equation}
        \inner{\nabla\cL(X)}{\Lambda(X)} \geq \frac12 \| \grad f(X)\|^2 + \nu \cN(X),
    \end{equation}
    for the choice of 
    \begin{equation}
        \mu \geq \frac{2}{3-4\varepsilon}\left( \Lc(1-\varepsilon) + 3s + \finalrevision{\hat{L}^2 \frac{(1+\varepsilon)^2}{\lambda(1-\varepsilon)}}  \right)\enspace,
    \end{equation}
    where $s = \sup_{X\in\stiefel^\varepsilon} \| \sym(X^\top \nabla f(X)) \|$, $\Lc>0$ is the smoothness constant of $f$ over $\stiefel^\varepsilon$, $\hat{L} = \max(L, L')$ with $L' = \max_{X\in\stiefel^\varepsilon}\|\nabla f(X)\|$,
    and $\varepsilon<\frac34$.
\end{proposition}
This demonstrates that $\mathcal{L}$ is in fact a merit function.
Indeed, the landing direction is an ascent direction for the merit function, since $\langle \nabla \mathcal{L}(X), \Lambda(X)\rangle \geq 0$.
We can then combine this proposition and Proposition~\ref{prop:bound_landing_norm} 
get the following result:
\begin{proposition}
    \label{prop:bound_scalar_with_norm}
    Under the same conditions as in Proposition~\ref{prop:lyapunov}, defining $\rho = \min(\frac12, \frac{\nu}{4\lambda^2(1+\varepsilon)})$, we have for $X\in\stiefel^\varepsilon$:
    $$
    \langle \Lambda(X), \nabla \mathcal{L}(X)\rangle \geq \rho \|\Lambda(X)\|^2.
    $$
\end{proposition}
We now turn to the intuition behind the merit function.
The merit function $\mathcal{L}$ is composed of three terms. The terms
$f(X)$ and $\mu\mathcal{N}(X)$ are easy to interpret: the first one controls optimality, while the second controls
the distance to the manifold. The term $h(X)$ might be mysterious at first. Its role is best understood when $X$ is on the manifold. Indeed, for $X\in\stiefel$ we get
$$
\nabla h(X) = -X\mathrm{sym}(X^\top \nabla f(X))\enspace.
$$
This vector is, in fact, the opposite of the
projection of $\nabla f(X)$ on the normal space to $\stiefel$ at $X$.
Hence, if $X\in \stiefel$ then $\mathcal{L}(X) = f(X)$ 
and $\nabla \mathcal{L}(X)$ is a vector in the tangent space $\tangent$ which is makes an acute angle with $\grad f(X)$. Note that Fletcher’s penalty function \citep{fletcher1970class} is similar to the merit function $\mathcal{L}(X)$ while the $h(X)$ term is determined by the solution of the least squares problem. Notice that this merit function $\mathcal{L}(X)$ is the same as that of \citet{gao2019parallelizable} where $\mathcal{L}$ is constructed from the augmented Lagrangian function and $h(X)$ serves as a multiplier term since the multiplier of orthogonality constraints has a closed-form solution, $\mathrm{sym}(X^\top \nabla f(X))$, at any first-order stationary point. The main difference with the present work is that  \citet{gao2019parallelizable} solve the optimization problem by taking steps in a direction that approximates $-\nabla \mathcal{L}(X)$, but that does not satisfy
the orthogonality hypothesis and hence is not guaranteed to converge for any value of $\lambda>0$ (see also the discussion in Appendix B in~\citep{ablin2022fast}).

As a sum of smooth terms, the merit function is also smooth:
\begin{proposition}[Smoothness of the merit function]
\label{prop:smoothness_merit}
The merit function $\mathcal{L}$ is $L_g$-smooth on $\stiefel^\varepsilon$, with $L_g = L_{f+h} + \mu L_{\mathcal{N}}$ where $L_{f+h}$ is the smoothness constant of $f + h$ and $L_{\mathcal{N}}$ is that of $\mathcal{N}$, upper bounded for instance by $2 + 3\varepsilon$.
\end{proposition}

\citet{schechtman2023orthogonal} consider instead the non-smooth merit function $\mathcal{L}'(X) = f(X) + M\|X^{\top}X-I_p\|$. Our merit function decreases faster in the direction normal to the manifold, which is why the term $h(X)$ is introduced to tame the contribution of $f$ in that direction. The smoothness of our merit function renders the subsequent analysis of the algorithms particularly simple since the smoothness lemma applied on $\mathcal{L}$ directly gives a descent lemma. This forms the basic tool to analyze the landing method and its variants.

\revision{This new merit function allows us to obtain far stronger convergence results than those of \cite{ablin2022fast}, which only analyzed the deterministic method, and could only prove a slow rate with decreasing step-sizes.
In the following, we demonstrate that the classical gradient descent, SGD, and variance reduction method can be straightforwardly extended to the landing setting, provably converge to the manifold, and recover the same convergence rates at their Riemannian counterparts, all thanks to this new merit function.}

\section{A Family of Landing Algorithms, and their Convergence Properties}
\label{sec:algorithms}
In this section, we consider a family of methods all derived from a base algorithm, the landing gradient descent algorithm. All of our algorithms follow directions of the form~\eqref{eq:skew_field}.

\subsection{Landing Gradient Descent: the Base Algorithm} \label{subsec:landing_gradient}

This algorithm produces a sequence of iterates $X_k \in \mathbb{R}^{n\times p}$ by iterating
\begin{equation}
  \label{eq:landing_algorithm}
      X_{k+1} = X_k  - \eta \Lambda(X_k)
\end{equation}
where $\eta>0$ is a step size and $\Lambda$ is the landing field defined in~\eqref{eq:landing_field}. Note that this method falls into the hypothesis of Proposition~\ref{prop:recursive_step_size} with the simple maps $\mathcal{A}_k(X_0, \dots, X_k) = \grad f(X_k)$, so we can just take $\tilde{a} = \sup_{X\in\stiefel^\varepsilon}\|\grad f(X)\|$ to get a safeguard step size $\eta^*$ that guarantees that the iterates of the landing algorithm stay in $\stiefel^\varepsilon$.

We will start with the analysis of this method, where we find that it achieves a rate of convergence of $\frac1K$: we have $\frac1K\sum_{k=0}^K\|\grad f(X_k)\|^2 = O(\frac 1K)$ and $\frac1K\sum_{k=0}^K\mathcal{N}(X_k) = O(\frac 1K)$.
We, therefore, obtain the same properties as classical Riemannian gradient descent, with a reduced cost per iteration, but with different constants.

\begin{proposition}
\label{prop:convergence_landing}
  Consider the iteration~\eqref{eq:landing_algorithm} starting from $X_0\in\stiefel^\varepsilon$. Define $\tilde{a} = \sup_{X\in \stiefel^\varepsilon}\|\sk(\nabla f(X)X^\top)X\|_F$, and let $\eta^*$ be the safeguard step size chosen from Lemma~\ref{lemma:safe_step size_lower}.  Let $\mathcal{L}^*$ be a lower bound of the merit function $\mathcal{L}$ on $\stiefel^\varepsilon$. Then, for $\eta \leq \min(\frac1{2L_g}, \frac{\nu}{4\lambda^2L_g(1+\varepsilon)}, \eta^*)$, we have 
  \begin{equation*}
  \frac1K\sum_{k=1}^K\|\grad f (X_k)\|^2\leq \frac{4(\mathcal{L}(X_0) -\mathcal{L}^*)}{\eta K} \enspace \text{ and } \enspace \frac1K\sum_{k=1}^K\mathcal{N}(X_k)\leq \frac{2(\mathcal{L}(X_0) -\mathcal{L}^*)}{\eta\nu K}.
  \end{equation*}
\end{proposition}

This result demonstrates weak convergence to the stationary points of $f$ on the manifold at a rate $\frac1{K}$, 
just like classical Riemannian gradient descent~\citep{boumal2019global}.
\finalrevision{Some readers might be more familiar with the following consequence ``without squares'': $\inf_{k\leq K}\|\grad f(X_k)\| +\|X_k^{\top}X_k-I_p\| =O(\frac{1}{\sqrt{K}})$, i.e., there exists an iterate which has both a low gradient norm and is close to the manifold.}
This result is an important improvement over that of \cite{ablin2022fast} since we do not require decreasing step sizes to get convergence and obtain a much better convergence rate.

\subsection{Landing Stochastic Gradient Descent: Large Scale Orthogonal Optimization}
\label{subsec:sgd}
We now consider the case where the function $f$ is the average of $N$ functions:
$$
f(X) = \frac1N\sum_{i=1}^Nf_i(X)\enspace .
$$
We can define the landing field associated with each $f_i$ by
$$
\Lambda_i(X) = \grad f_i(X) + \lambda X(X^{\top}X -I_p)\enspace .
$$
This way, we have
$$
\Lambda(X)= \frac1N\sum_{i=1}^N\Lambda_i(X),
$$
and if we take an index $i$ uniformly at random between $1$ and $N$ we have 
$$
\mathbb{E}_i[\Lambda_i(X)] = \Lambda(X)\enspace.
$$
In other words, the direction $\Lambda_i$ is an unbiased estimator of the landing field $\Lambda$.
We consider the landing stochastic gradient descent (Landing-SGD) algorithm, which at iteration $k$ samples a random index $i_k$ uniformly between $1$ and $N$ and iterates
\begin{equation}
    \label{eq:stochastic_landing}
    X_{k+1} = X_k - \eta_k \Lambda_{i_k}(X_k)
\end{equation}
where $\eta_k$ is a sequence of step size.
As is customary in the analysis of stochastic optimization algorithms, we posit a bound on the variance of $\Lambda_i$:

\begin{assumption}
There exists $B> 0$ such that for all $X\in\stiefel^\varepsilon$, we have 
$\frac1N\sum_{i=1}^N\|\Lambda_i(X) -\Lambda(X)\|^2\leq B$.
\end{assumption}
This assumption is true when the $\Lambda_i$ are continuous since $\stiefel^\varepsilon$ is a compact set.
\revision{Lemma \ref{lemma:safe_step size_lower} and Prop.~\ref{prop:recursive_step_size} come in handy to define a safeguard step size.
Indeed we see that the algorithm follows the hypothesis of Prop.~\ref{prop:recursive_step_size} with 
$\tilde{a} = \sup_{X\in \stiefel^\varepsilon, \enspace i \in\{1,\dots, n\}}\|\grad f_i(X)\|$. 
This allows to define a safeguard step-size $\eta^*$ following Prop.~\ref{prop:recursive_step_size}.
}
We then obtain a simple recursive bound on the iterates using the smoothness inequality:
\begin{proposition}
\label{prop:decrease_stochastic}
  Assume that $\eta_k\leq \min(\frac1{2L_g}, \frac{\nu}{4\lambda^2L_g(1+\varepsilon)}, \eta^*)$ where $\eta^*$ is the global safeguard step size obtained as above. Then, 
  $$
  \mathbb{E}_{i_k}[\mathcal{L}(X_{k+1})] \leq \mathcal{L}(X_k) - \frac{\eta_k}4 \|\grad f(X_k)\|^2 - \frac{\eta_k\nu}2\mathcal{N}(X_k) + \frac{L_g B\eta_k^2}2,
  $$
  where the expectation is taken with respect to the random variable $i_k$.
\end{proposition}

We get convergence rates of the stochastic landing algorithm with decreasing step sizes. 

\begin{proposition}
\label{prop:convergence_stochastic_decreasing_step}
  Assume that the step size is $\eta_k = \eta_0 \times (1 + k)^{-\frac12}$ with $\eta_0 = \min(\frac1{2L_g}, \frac{\nu}{4\lambda^2L_g(1+\varepsilon)}, \eta^*)$, with $\eta^*$ chosen as the safeguard step size.
  Then, we have 
  $$
  \inf_{k\leq K}\mathbb{E}[\|\grad f(X_k)\|^2] =O\left(\frac{\log(K+1)}{\sqrt{K}}\right)\enspace \text{ and } \enspace \inf_{k\leq K}\mathbb{E}[\|\mathcal{N}(X_k)\|^2] =O\left(\frac{\log(K+1)}{\sqrt{K}}\right) \enspace.
  $$
  The expectation here is taken with respect to all the random realizations of the random variables $i_k, \enspace k\leq K$.
\end{proposition}

This shows that our method with decreasing step size has the same convergence rate as Riemannian stochastic gradient descent with decreasing step size.
With constant step size, we get the following proposition:

\begin{proposition}
\label{prop:convergence_stochastic_fixed_step}
  Assume that the step size is fixed to $\eta = \eta_0 \times (1 + K)^{-\frac12}$ with $\eta_0 = \min(\frac1{2L_g}, \frac{\nu}{4\lambda^2L_g(1+\varepsilon)}, \eta^*)$.
  Then, we have 
  $$
  \inf_{k\leq K}\mathbb{E}[\|\grad f(X_k)\|^2] =O\left(\frac{1}{\sqrt{K}}\right)\enspace \text{ and } \enspace \inf_{k\leq K}\mathbb{E}[\|\mathcal{N}(X_k)\|^2] =O\left(\frac{1}{\sqrt{K}}\right)\enspace.
  $$
\end{proposition}
\paragraph{Sample complexity}
The sample complexity of the algorithm is readily obtained from the bound: in order to find an $\varepsilon$-critical point of the problem and get both $\inf_{k\leq K}\|\grad f(X_k)\|^2 \leq \varepsilon$ and $\inf_{k\leq K}\mathcal{N}(X_k)\leq \varepsilon$, we need $O(\varepsilon^{-2})$ iterations.
The $O$ here only hides constants of the problem, like conditioning of $f$ and hyperparameter $\lambda$, but this quantity is independent of the number of samples $N$.
This matches the classical sample complexity results obtained with SGD in the Euclidean and Riemannian non-convex settings\finalrevision{~\citep{zhang2016riemannian}}.

\subsection{Landing SAGA: Variance Reduction for Faster Convergence}
\label{subsec:SAGA}

In this section, we are in the same finite-sum setting as in Section~\ref{subsec:sgd}.
As in classical optimization, SGD suffers from the high variance of its gradient estimator, leading to sub-optimal convergence rates.
A classical strategy to overcome this issue consists in using variance reduction algorithms, which build an estimator of the gradient whose variance goes to $0$ as training progresses.
Such algorithms have also been proposed in a Riemannian setting, but like most other methods, they also require retractions~\citep{zhang2016riemannian}.

We propose a retraction-free variance-reduction algorithm that is a crossover between the celebrated SAGA algorithm~\citep{defazio2014saga} and the landing algorithm, called the landing SAGA algorithm.
The algorithm keeps a memory of the last gradient seen for each sample, $\Phi^1_k, \dots, \Phi_k^N$ where each $\Phi_k^i \in \mathbb{R}^{n\times p}$. At iteration $k$, we sample at random an index $i_k$ between $1$ and $N$, and compute the direction
$\Lambda^{i_k}_k = \grad f_{i_k}(X_k) - \sk(\Phi_k^{i_k}X_k^\top)X_k + \frac1N \sum_{j=1}^N\sk(\Phi_k^j X_k^\top)X_k+\lambda X_k(X_k^\top X_k - I_p)$. We update the memory corresponding to sample $i_k$ by doing $\Phi_{k+1}^{i_k} = \nabla f_{i_k}(X_k)$, and $\Phi_{k+1}^j = \Phi_{k}^j$ for all $j\neq i_k$. We then move in the direction
$$
X_{k+1} = X_k -\eta \Lambda_k^{i_k}.
$$
It is important to note that variance reduction is only applied on the ``Riemannian'' part $\grad f_i(X)$. The other term $X(X^\top X- I_p)$ is treated as usual. Like in the classical SAGA algorithm, we have the unbiasedness property:
$$
\mathbb{E}_i[\Lambda_k^i] = \Lambda(X_k)\enspace.
$$
This means that, on average, the direction we take is the landing field, computed over the whole dataset.
The gist of this method is that we can have fine control on $\mathbb{E}_i[\|\Lambda^k_i\|^2]$. Indeed, letting $D_k^i = \grad f_i(X_k) -\sk(\Phi_k^iX_k^\top)X_k$, we have
\finalrevision{
\begin{equation}
\Lambda_k^i=\Lambda(X_k) + \underbrace{D_k^i - \mathbb{E}_j[D_k^j]}_{\text{zero-mean}} ,
\end{equation}
so that a bias-variance decomposition gives}
\begin{align}
  \mathbb{E}_i[\|\Lambda^i_k\|^2] &= \|\Lambda(X_k)\|^2 + \mathbb{E}_i[\|D_k^i - \mathbb{E}_j[D_k^j]\|^2]\\
  &\leq \|\Lambda(X_k)\|^2 + \mathbb{E}_i[\|D_k^i\|^2],
\end{align}
and $D_k^i$ can also be controlled since
\begin{align}
     \|D_k^i\|&= \|\sk((\nabla f_i(X_k) - \Phi_k^i)X_k^\top)X_k\| \\
     &\leq(1+\varepsilon)\|\nabla f_i(X_k) - \Phi_k^i\|\\
     &\leq(1+\varepsilon) L_f \|X_k - X_k^i\|^2
\end{align}
where $X_k^i$ is the last iterate that what chosen for the index $i$ (so $X_k^i$ is such that $\nabla f(X_k^i) = \Phi_k^i$). We, therefore, recover that we need to control the distance from the memory $X_k^i$ to the current point $X_k$, as is customary in the analysis of SAGA.

We have the following convergence theorem, which is obtained by combining the merit function and the proof technique of~\citet{reddi2016fast}:
\begin{proposition}
\label{prop:convergence_saga}
  Define $\rho$ as in Proposition~\ref{prop:bound_scalar_with_norm}, $L_g$ as in Proposition~\ref{prop:smoothness_merit} and $L_f$ the smoothness constant of $f$ on $\stiefel^\varepsilon$. Assume that the step size is such that  
  $$\eta\leq \min\left(\eta^*, \frac{\rho}{L_g}, \frac1{\sqrt{8N(1+\varepsilon)}L_f}, \left(\frac \rho{8N(4N+2)L_gL_f^2(1+\varepsilon)^2}\right)^{1/3}\right)\enspace.$$
  Then, we have 
  $$
  \inf_{k\leq K}\mathbb{E}[\|\grad f(X_k)\|^2] =O\left(\frac{1}{\eta K}\right)\enspace \text{ and } \enspace \inf_{k\leq K}\mathbb{E}[\|\mathcal{N}(X_k)\|^2] =O\left(\frac{1}{\eta K}\right)\enspace .
  $$
\end{proposition}
As in classical optimization, using the variance reduction of SAGA recovers a $\frac1K$ rate with a stochastic algorithm.

\paragraph{Sample complexity}
When the number of samples $N$ is large, the last term in the above ``$\min$'' for the choice of step size is the smallest; hence the step size scales as $N^{-2/3}$. This shows that to get to a $\varepsilon-$critical point such that $\|\grad f(X)\|^2\leq\varepsilon$, we need $O(N^{\frac23}\varepsilon^{-1})$ iterations. 
This matches the sample complexity of classical Euclidean SAGA in the non-convex setting~\citep{reddi2016fast}.

\subsection{Comparison to Penalty Methods}
\label{sec:penalty}

It is a common practice in deep learning applications that the orthogonality is favored by adding an $\ell_2$ regularization term and minimizing
\begin{equation*}
    f(X) + \lambda\mathcal{N}(X),
\end{equation*}
for example in \citep{balestriero2018spline, xie2017all, bansal2018can}.
This method leads to a small computational overhead compared to the simple unconstrained minimization of $f$, and it allows the use of standard optimization algorithms tailored for deep learning. However, it provides no guarantee that the orthogonality will be satisfied. Generally, there are two possible outcomes based on the choice of $\lambda>0$. If $\lambda$ is small, then the final point is far from orthogonality, defeating the purpose of the regularization. If $\lambda$ is too large, then optimization becomes too hard, as the problem becomes ill-conditioned: its smoothness constant grows with $\lambda$ while its strong convexity constant does not, since $\mathcal{N}(X)$ is not strongly convex (indeed, it is constant in the direction tangent to $\stiefel$).

In order to have a more formal statement than the above intuition, we consider the simple case of a linear problem:
\begin{proposition}
    \label{prop:linear_problem_with_constraint}
    Let $M = U\Sigma V^\top$ be a
    singular value decomposition of $M$, where $U\in\stiefel$, $\Sigma$ is a diagonal matrix of positive entries, and $V\in \mathrm{St}(p, p)$ is an orthogonal matrix. Let $\sigma_1\leq\dots\leq \sigma_p$ be the singular values of $M$. Then, the minimizer of $g(X) = \langle M, X\rangle + \lambda \mathcal{N}(X)$ is $X^* = - U \Sigma^*V^{\top}$ where $\Sigma^*$ is the diagonal  matrix of entries $\sigma^*_i$ where $\sigma^*_i$ is the minimizer of the scalar function $x\to -\sigma_i x + \frac\lambda4(x^2 - 1)^2$. Furthermore, we have two properties:
    \begin{itemize}
        \item The distance between $X^*$ and the constrained solution $X_{\stiefel} = UV^{\top}$ is 
        of the order of $\lambda^{-1}$ when $\lambda$ goes to $+\infty$.
        \item The maximal and minimal eigenvalues of the Hessian $H$ of $g$ at $X^*$ satisfy $\lambda_{\min}(H) \leq \sigma_p + \frac{\sigma_p^2}{4\lambda}$ and $\lambda_{\max}(H)\geq 2\lambda$. Hence the conditioning of $H$ is at least $\frac{2\lambda}{\sigma_p + \frac{\sigma_p^2}{4\lambda}}$ which behaves like
        $\frac{2\lambda}{\sigma_p}$ as $\lambda$ goes to $+\infty$. 
    \end{itemize}
\end{proposition}
\begin{figure}[t]
    \centering
        \includegraphics[width=.8\textwidth]{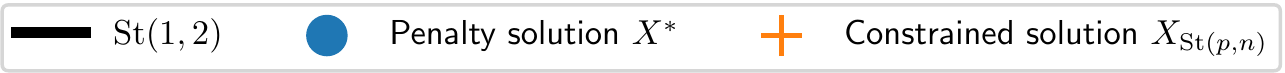}\\
    \includegraphics[width=.32\textwidth]{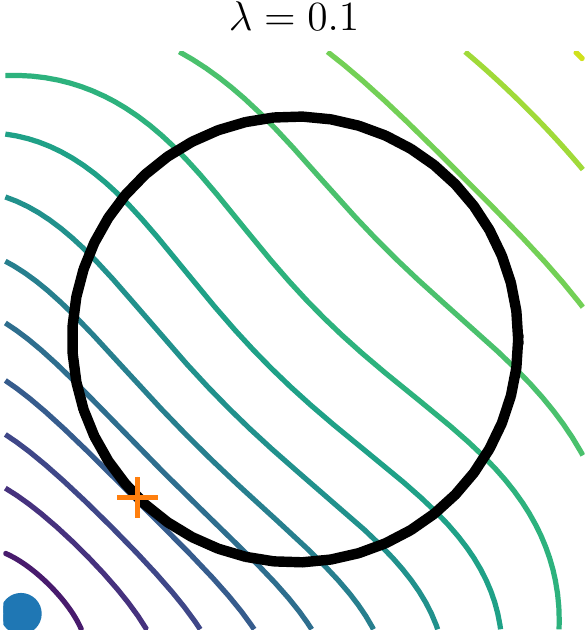}
    \includegraphics[width=.32\textwidth]{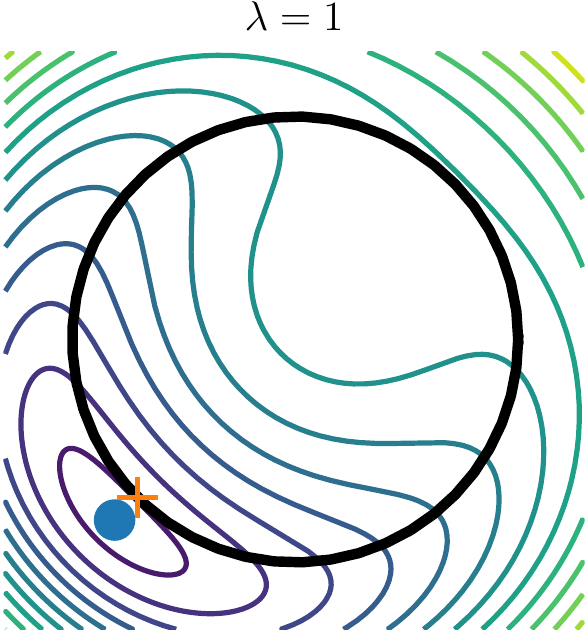}
    \includegraphics[width=.32\textwidth]{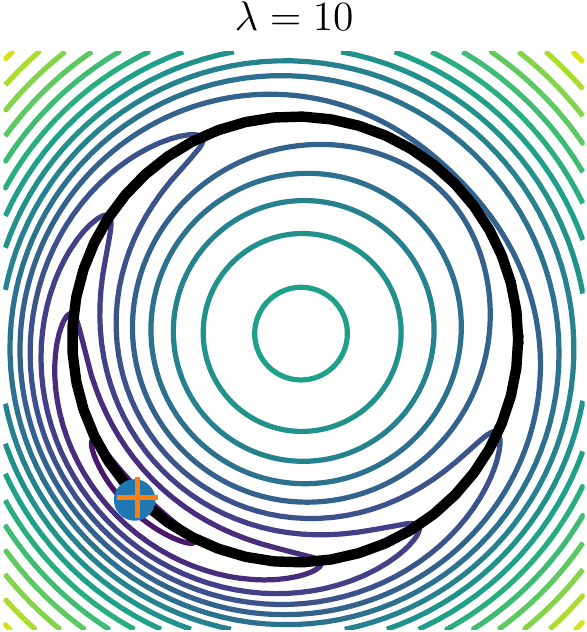}
    \caption{Contours of the function $g(X)  = f(X) + \lambda \mathcal{N}(X)$ in dimensions $(n, p) = (2, 1)$. When $\lambda$ is small, its minimizer is far from $\stiefel$, and when $\lambda$ is large, $g$ is badly conditioned around its optimum.  \label{fig:penalty}}
\end{figure}
This unfortunate trade-off is illustrated in Figure~\ref{fig:penalty}, where we display the contours of $g$ for different values of $\lambda$.
In practical terms, these two points mean that :
\begin{itemize}
    \item In order to get a solution close to the correct one $X_{\stiefel}$, we have to take a large $\lambda$.
    \item But, taking a large $\lambda$ makes the conditioning of the problem go to infinity. The number of iterations of gradient descent with fixed step-size on $g$ to achieve $\|X - X^*\|\leq \epsilon$ is locally of the order $\frac{2\lambda}{\sigma_p}\log(\epsilon^{-1})$, which is linear in $\lambda$.
\end{itemize}
In order to better formalize the second point, consider that we start gradient descent from a point $X = X^* +\delta E$ where $\delta \ll 1$ and $E$ is an eigenvector of $H$ corresponding to its smallest eigenvalue, i.e., so that $H[E] = \lambda_{\min}(H)E$. We take the step size as the inverse of the local smoothness constant $\eta =1/\lambda_{\max}(H)$, and define $Y = X  -\eta \nabla g(X)$ the output of one iteration of gradient descent. Then, as $\delta \to 0$, we find 
$$
Y - X^* = (1 - \eta \lambda_{\min}(H))(X - X^*) + o(\delta)
$$
which means that we have (approximately in $\delta$) a linear convergence toward $X^*$ with a rate $(1 - \eta \lambda_{\min}(H)) \geq 1 -\frac{\sigma_p}{2\lambda} +o(\frac1\lambda)$. Hence, after $K$ iterations of gradient descent, the error is at least
$(1 -\frac{\sigma_p}{2\lambda})^K$, and having an error smaller than $\epsilon$ requires at least $K \geq \frac{\log(\epsilon^{-1})}{\log(1 - \frac{\sigma_p}{2\lambda})} = \frac{2\lambda}{\sigma_p}\log(\epsilon^{-1}) + o(\frac1\lambda)$ iterations: we need a number of iterations proportional to $\lambda$.

\revision{We want to insist that this behavior is in stark contrast with that of the landing methods presented above since they all provably converge \emph{to the manifold} regardless of the choice of $\lambda$.
In terms of computational overhead, they only require an additional computation of a Riemannian gradient (see \autoref{eq:riemannian_gradient}).
We argue that this might be a very small price to pay in front of the benefits of having a method that provably converges to the correct points.}

\revision{
    \subsection{Computational cost}
    We now analyze the computational cost of one iteration of the landing methods presented above and compare it to the cost of classical Riemannian methods.
    All methods require computing the (perhaps stochastic) gradient of $f$, which we denote for short as $G\in\mathbb{R}^{n\times p}$. 
    We denote $t_G$ as the cost of this computation, which depends on the complexity of the function $f$ itself. In the following, we use the fact that the cost of multiplying an $a \times b$ matrix with a $b\times c$ matrix is $O(abc)$.
    \paragraph{Penalty method}
    The penalty method needs to compute $\nabla(f(X) + \lambda \mathcal{N}(X)) = G + \lambda X(X^\top X - I_p)$. Hence, on top of computing $G$, it requires two matrix-matrix multiplications; and its overall cost is $t_G + O(np^2)$.
    We recall that this method does not converge to the stationary points of the correct problem as illustrated in \autoref{sec:penalty}.
    \paragraph{Landing methods} The landing methods presented above then compute the direction $\Lambda(X) = \mathrm{skew}(GX^\top)X + \lambda X(X^\top X - I_p)$.
    Since $n\geq p$, the ordering of operations to compute this quantity that leads to the optimal complexity is:
    \begin{itemize}
        \item \makebox[15em]{$A = X^\top X$\hfill}  \makebox[10em]{\texttt{Shape: p, p \hfill}} \makebox[10em]{\texttt{Cost: $O(np^2)$\hfill}}

        \item \makebox[15em]{$B = (\frac12 G + \lambda X)A$\hfill}  \makebox[10em]{\texttt{Shape: n, p \hfill}} \makebox[10em]{\texttt{Cost: $O(np^2)$\hfill}}

        \item \makebox[15em]{$C = G^{\top}X$\hfill}  \makebox[10em]{\texttt{Shape: p, p \hfill}} \makebox[10em]{\texttt{Cost: $O(np^2)$\hfill}}
        
        \item \makebox[15em]{$D = XC$\hfill}  \makebox[10em]{\texttt{Shape: n, p \hfill}} \makebox[10em]{\texttt{Cost: $O(np^2)$\hfill}}

        \item \makebox[15em]{Return $\Lambda = B - \frac12 D - \lambda X$\hfill}  \makebox[10em]{\texttt{Shape: n, p \hfill}} \makebox[10em]{\texttt{Cost: $O(np)$\hfill}}
    \end{itemize}
    This makes it clear that the cost of implementing the landing method is $t_G + O(np^2)$ in terms of time and $O(np)$ in terms of memory. It requires only four matrix multiplications, which are heavily parallelizable. Note that compared to the penalty method, it requires only two more matrix multiplications. 
    Interestingly, when $n \simeq p$ and we can afford to form $n\times n$ matrices, we can simply compute the field as $\Lambda = \left(\mathrm{skew}(GX^\top) +\lambda (XX^\top - I_n)\right)X$, which only requires three matrix multiplications.
    \paragraph{Retraction-based methods}
    Retraction-based methods have to compute costly retractions, as described in \autoref{subsec:riem_optim}.
    The cost of these computations is usually $O(np^2)$, but they are much slower than matrix multiplications when $n\simeq p$. Hence, the overall cost has the same order of magnitude as the landing methods, $t_G + O(np^2)$, but with much worse constants, making overall one landing iteration faster to compute. 
    The cost of computing the gradient plays a key role here: indeed, if $t_G$ is much greater than $np^2$, then the cost of both methods becomes very similar. 
}

\section{Experiments} \label{sec:experiments}

We numerically compare the landing method against the two main alternatives, the Riemannian gradient descent with QR retraction and the Euclidean gradient descent with added $\ell_2$ squared penalty norm, with stochastic gradients.\footnote{The code to reproduce the experiments in this section is publicly available at: \url{https://github.com/simonvary/landing-stiefel}.}
\finalrevision{In all experiments, we take the safe region parameter to be $\varepsilon = \frac12$. Unless specified otherwise, we use for the landing term $\lambda =1$, and choose the learning rate $\eta$ with a grid search, just like for all other methods.}

\subsection{Online PCA}
\begin{figure}[t]
    \centering
    \begin{subfigure}[b]{0.9\textwidth}
        \includegraphics[width=.49\textwidth]{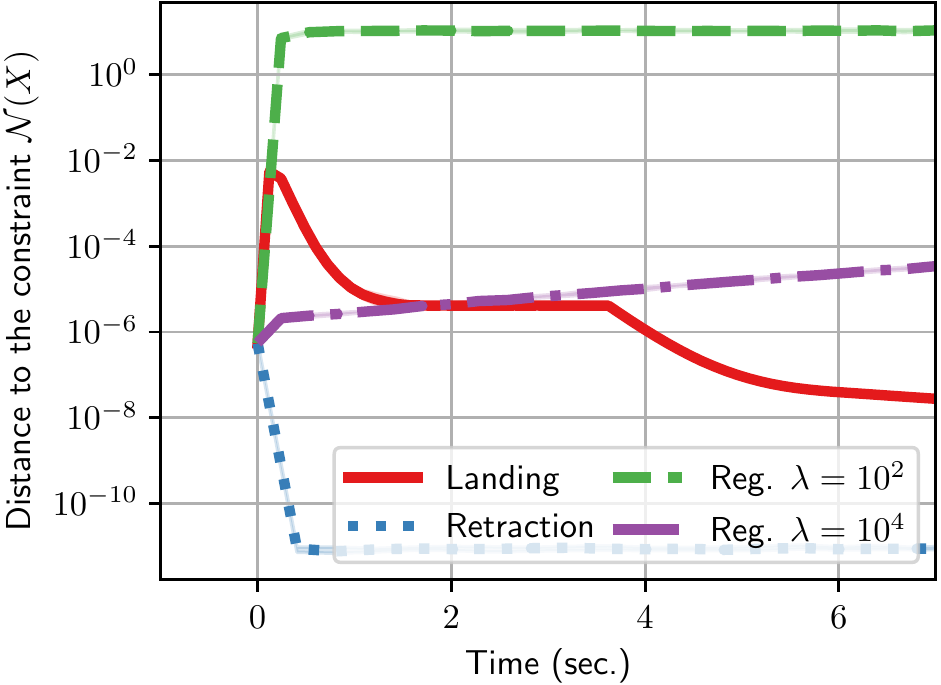}
        \includegraphics[width=.47\textwidth]{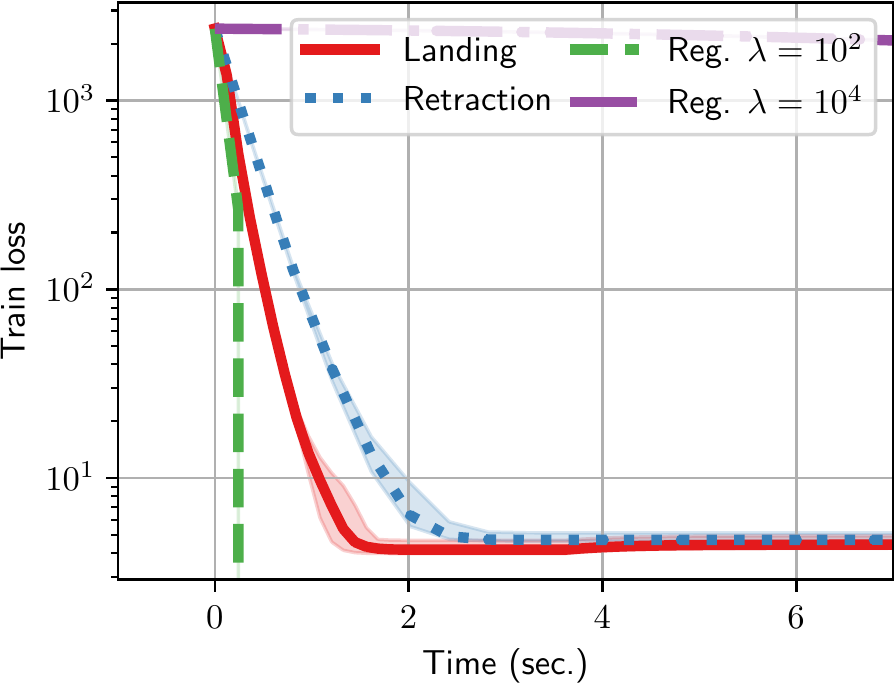}
        \caption{$n=5000, p=200$ \label{fig:pca_p500}}
    \end{subfigure}
    \begin{subfigure}[b]{0.9\textwidth}
        \includegraphics[width=.49\textwidth]{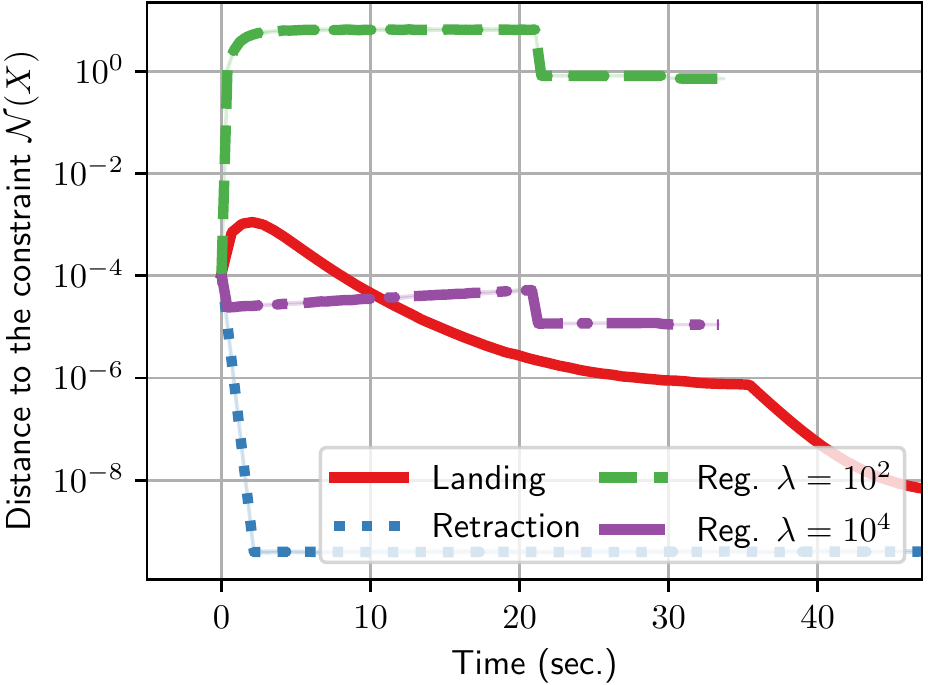}
        \includegraphics[width=.47\textwidth]{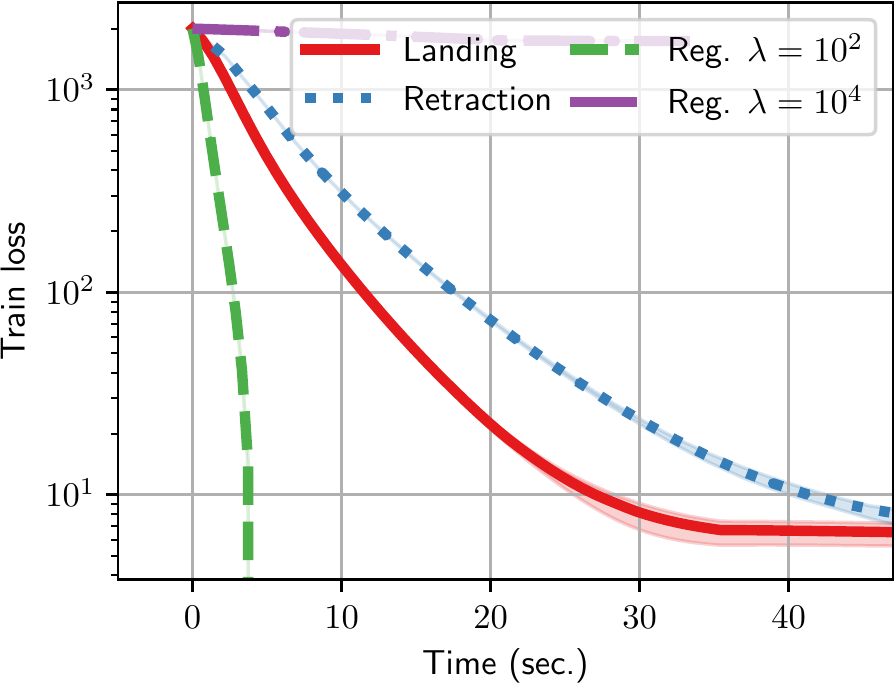}
        \caption{$n=5000, p=1000$ \label{fig:pca_p1000}}
    \end{subfigure}
    \caption{Experiments with online PCA. \label{fig:pca}}
\end{figure}
We test the methods performance on the online principal component analysis (PCA) problem 
\begin{equation}
    \min_{X\in\mathbb{R}^{n\times p}} -\frac{1}{2}\left\| AX \right\|_F^2, \quad \text{s.t.}\quad X\in\stiefel,
\end{equation}
where $A\in\mathbb{R}^{N\times n}$ is a synthetically generated data matrix with $N=15\,000$ being the number of samples each with dimension $n=5000$. The columns of $A$ are independently sampled from the normal distribution $\mathcal{N}(0, UU^\top + \sigma I_n)$, where $\sigma = 0.1$ and $U\in\mathbb{R}^{n\times p}$ is sampled from the Stiefel manifold with the uniform Haar distribution.

We compare the landing stochastic gradient method with the classical Riemannian gradient descent and with the ``penalty'' method which minimizes $f(X) + \lambda \cN(X)$, where $\lambda$ is now a regularization hyperparameter, using standard SGD.
Figure \ref{fig:pca} 
shows the convergence of the objective and the distance to the constraint against the computation time of the three methods using stochastic gradients with a batch size of $128$ and a fixed step size, which decreases after $30$ epochs. The training loss is computed as $f(X_k) - f(X_*)$, where $X_*$ is the matrix of $p$ right singular vectors of $A$. We see that in both cases of $p = 200$ and $p=1000$ the landing is able to reach a lower objective value faster compared to the Riemannian gradient descent, however, at the cost of not being precisely on the constraint but with $\cN(X)\leq 10^{-6}$. The distance is further decreased after 30 epochs as the fixed step size of the iterations is decreased as well. This test is implemented in PyTorch and performed using a single GPU. 

Euclidean gradient descent with $\ell_2$ regularization  performs poorly with both choices of regularizer $\lambda$ (``Reg.'' in the figure). In the first case when $\lambda=10^2$ is too small, the distance remains large, and as a result, the training loss becomes negative since we are comparing against $X_*$ which is on the constraint. In the second case of the large penalty term, when $\lambda = 10^4$, the iterates remain relatively close to the constraint, but the convergence rate is very slow. These experimental findings are in line with the theory explained in Section~\ref{sec:penalty}. 

In general, we see that the landing method outperforms Riemannian gradient descent in cases when the computational cost of the retraction is more expensive relative to computing the gradient. This occurs especially in cases when $p$ is large or the batch size of the stochastic gradient is small, as can be seen also in the additional experiments for $p=100$ and $p=500$ shown in Figure \ref{fig:pca_app} in Appendix~\ref{app:numerics}.

\subsection{Independent Component Analysis}
\begin{figure}[t]
    \centering
        \includegraphics[width=.99\textwidth]{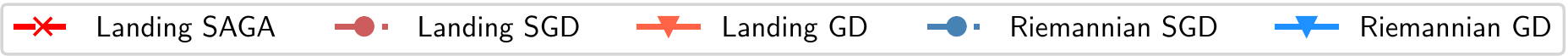}\\
    \includegraphics[width=.32\textwidth]{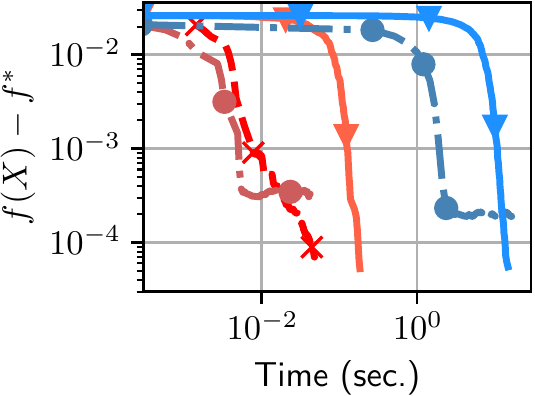}
    \includegraphics[width=.32\textwidth]{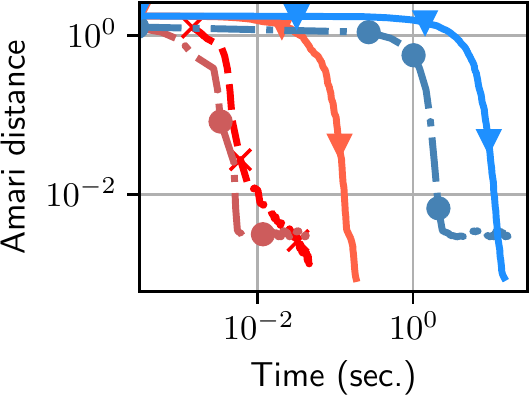}
    \includegraphics[width=.32\textwidth]{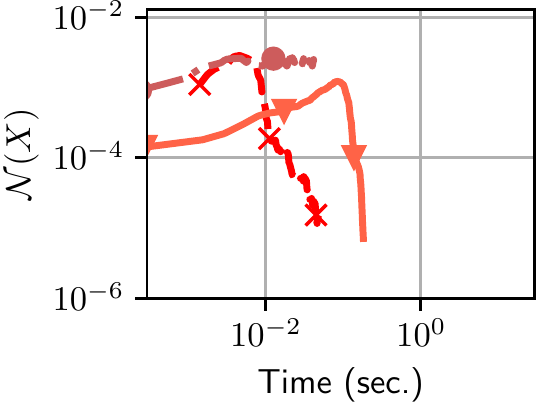}
    \caption{Experiments on ICA. Both scales are logarithmic. \label{fig:ica}}
\end{figure}
Given a data matrix $A = [a_1,\dots, a_N] \in \mathbb{R}^{N\times n}$, we perform the ICA of $A$ by solving~\citep{hyvarinen1999fast}
$$
\min_{X\in\mathbb{R}^{n\times n }} \frac 1 N \sum_{i=1}^N\sum_{j=1}^n \sigma([AX]_{ij}), \text{ such that } X\in\mathrm{St}(n, n)
$$
where $\sigma$ is a scalar function defined by $\sigma(x) = \log(\cosh(x))$, so that $\sigma'(x) = \tanh(x)$.

We generate $A$ as $A = S B^{\top}$, where $S\in\mathbb{R}^{N\times n}$ is the sources matrix, containing i.i.d. samples drawn from a Laplace distribution, and $B$ is a random orthogonal $n \times n$ mixing matrix. We take $n = 10$ and $N = 10^4$. We run the base landing algorithm, the landing SGD, and the landing SAGA algorithm, with a batch size of $100$ and compare them with Riemannian gradient descent and Riemannian SGD. 
Additionally to the loss and distance to the manifold, we also record the Amari distance, which measures how well $B^{\top}X$ is a scale and permutation matrix, i.e., how well we recover the sources.
Results are displayed in Figure~\ref{fig:ica}. We observe a fast convergence for the SGD algorithms followed by a plateau. The landing SAGA algorithm overcomes this plateau thanks to variance reduction. It is also much faster than the full-batch algorithm since it is a stochastic algorithm. We finally observe that its distance to the manifold reaches a very low value. Here, Riemannian methods are much slower than the landing because of the costly retractions.
This experiment was run on a CPU using Benchopt~\citep{benchopt}.
\subsection{Neural Networks with Orthogonality Constraints}
\begin{figure}[t]
    \centering
    \begin{subfigure}[b]{0.9\textwidth}
        \includegraphics[width=.49\textwidth]{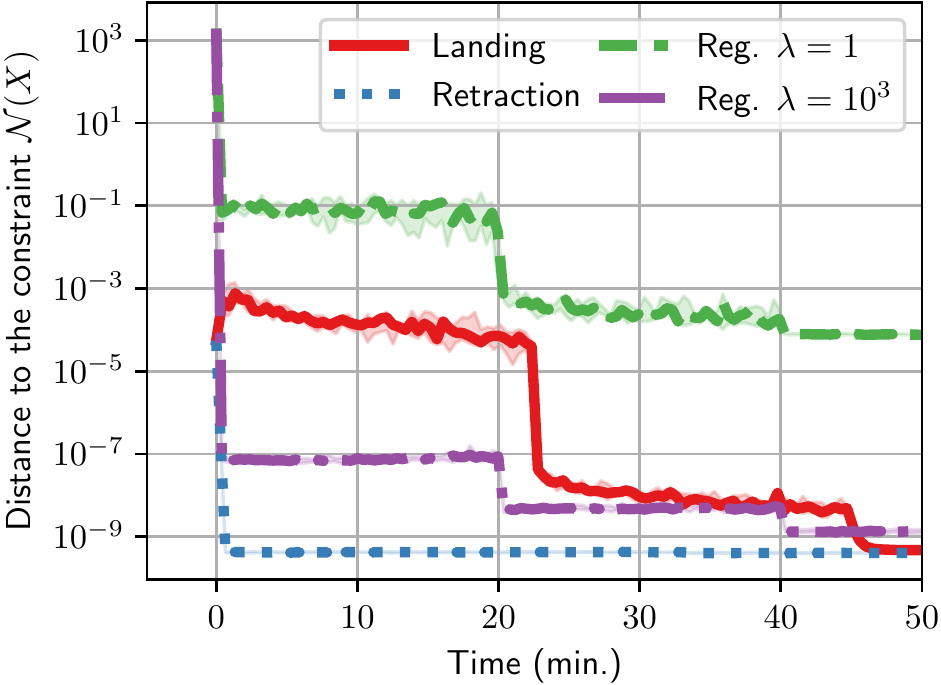}
        \includegraphics[width=.49\textwidth]{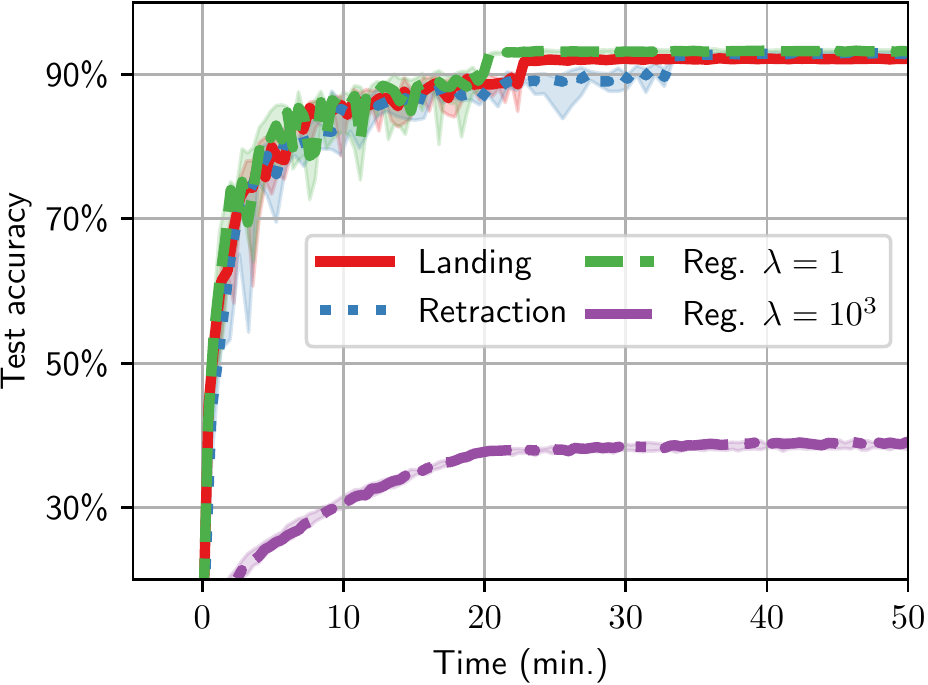}
        \caption{Resnet18 experiments \label{fig:ort_conv_resnet}}
    \end{subfigure}
    \begin{subfigure}[b]{0.9\textwidth}
        \includegraphics[width=.49\textwidth]{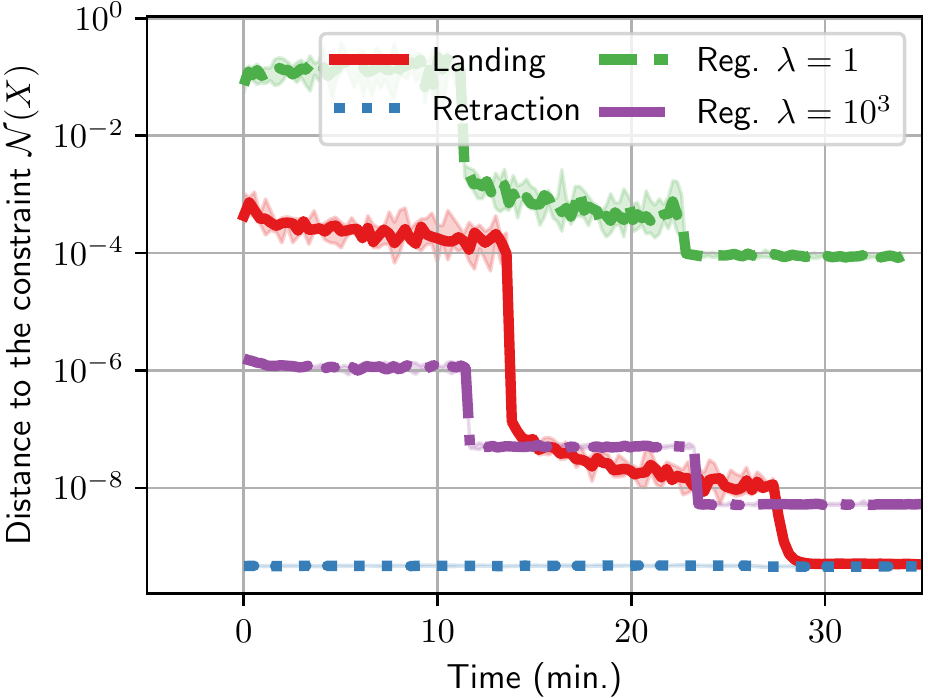}
        \includegraphics[width=.49\textwidth]{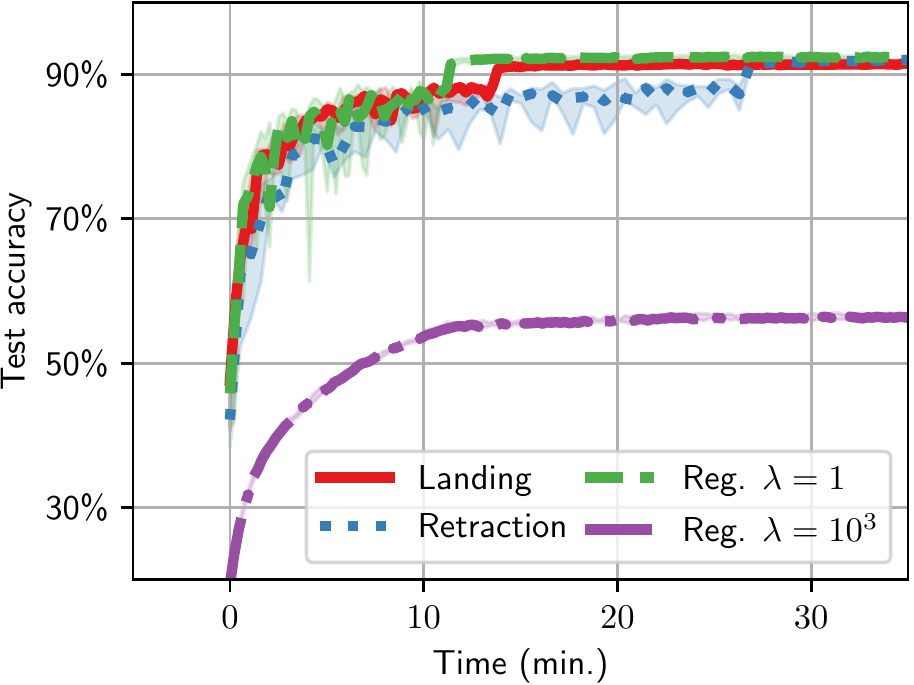}
        \caption{VGG16 experiments \label{fig:ort_conv_vgg16}}
    \end{subfigure}
    \caption{Experiments with orthogonal convolutions. \label{fig:ort_conv}}
\end{figure}
We test the methods for training neural networks whose weights are constrained to the Stiefel manifold. Orthogonality of weights plays a prominent role in deep learning, for example in the recurrent neural networks to prevent the problem of instable gradients~\citep{arjovsky2015unitary}, or in orthogonal convolutional neural networks  that impose kernel orthogonality to improve the stability of models~\citep{bansal2018can, wang2020orthogonal}.

We perform the test using two standard models, VGG16~\citep{simonyan2014very} and Resnet18~\citep{he2016deep}, while constraining the kernels of all convolution layers to be on the Stiefel manifold. We reshape the convolutional kernels to the size $n_\mathrm{out} \times n_\mathrm{in} n_\mathrm{x} n_\mathrm{y}$, where $n_\mathrm{in}, n_\mathrm{out}$ is the number of input and output channels respectively and $n_\mathrm{x}, n_\mathrm{y}$ is the filter size. In the case when the reshaping results in a wide instead of a tall matrix, we impose the orthogonality on its transposition. We train the models using Riemannian gradient descent, Euclidean gradient descent with $\ell_2$ regularization, and the landing method, with batch size of $128$ samples for $150$ epochs, and with a fixed step size that decreases as $\eta = \eta/10$ every 50 epochs. We repeat each training $5$ times for different random seed. This test is implemented in PyTorch and performed using a single GPU. 

Figure~\ref{fig:ort_conv} shows the convergence of the test accuracy and the sum of distances to the constraints against the computation time, with the light shaded areas showing minimum and maximum values of the $5$ runs. %
The figure shows the landing is a strict improvement over the Euclidean gradient descent with the added $\ell_2$ regularization, which, for the choice of $\lambda = 1$, achieves a good test accuracy, but at the cost of the poor distance to the constraint of $10^{-3}$, and for the choice of $\lambda = 10^3$ converges to a solution that has similar distance as the landing of $10^{-8}$, but has poor test accuracy. In comparison, training the models with the Riemannian gradient descent with QR retractions, achieves the lowest distance to the constraint, but also takes longer to reach the test accuracy of roughly $90\%$.

We also compared with the trivialization approach \citep{lezcano2019trivializations} using the Geotorch library, however this approach is not readily suitable for optimization over large Stiefel manifolds. See the experiments in Figure~\ref{fig:ort_conv_app} in the Appendix \ref{app:numerics}, which takes over 7 hours, i.e. approximately 14 times as long as the other methods, to reach the test accuracy of around $90\%$ with VGG16.
\section*{Conclusion}
We have extended the landing method of~\cite{ablin2022fast} from the orthogonal group to the Stiefel manifold, yielding an iterative method for smooth optimization problems where the decision variables take the form of a rectangular matrix constrained to be orthonormal. The iterative method is infeasible in the sense that orthogonality is not enforced at the iterates. We have obtained a computable bound on the step size ensuring that the next iterate stays in a safe region. This safeguard step size, along with the smooth merit function~\eqref{eq:merit_function}, has allowed for a streamlined complexity analysis in Section~\ref{sec:algorithms}, both for the deterministic and stochastic cases. The various numerical experiments have illustrated the value of the proposed approach.

\acks{}
The authors thank Gabriel Peyré for fruitful discussions.
This work was supported by the Fonds de la Recherche Scientifique -- FNRS and the Fonds Wetenschappelijk Onderzoek -- Vlaanderen under EOS Project no 30468160, and by the Fonds de la Recherche Scientifique -- FNRS under Grant no T.0001.23. Simon Vary is a beneficiary of the FSR Incoming Post-doctoral Fellowship.
\newpage

\vskip 0.2in
\bibliography{sample}

\newpage
\appendix

\section{Further numerical experiments}\label{app:numerics}
Figure~\ref{fig:pca_app} shows the convergence plots of the landing method, Riemannian stochastic gradient descent (RGD) with QR retractions, and the stochastic gradient descent with the $\ell_2$ penalty (marked as ``Reg'') applied to the online PCA as described in Section~\ref{sec:experiments}. The landing compared to the stochastic RGD converges faster to the critical point in terms of computational time, especially with larger $p=500$.
\begin{figure}[h]
    \centering
    \begin{subfigure}[b]{0.9\textwidth}
        \includegraphics[width=.49\textwidth]{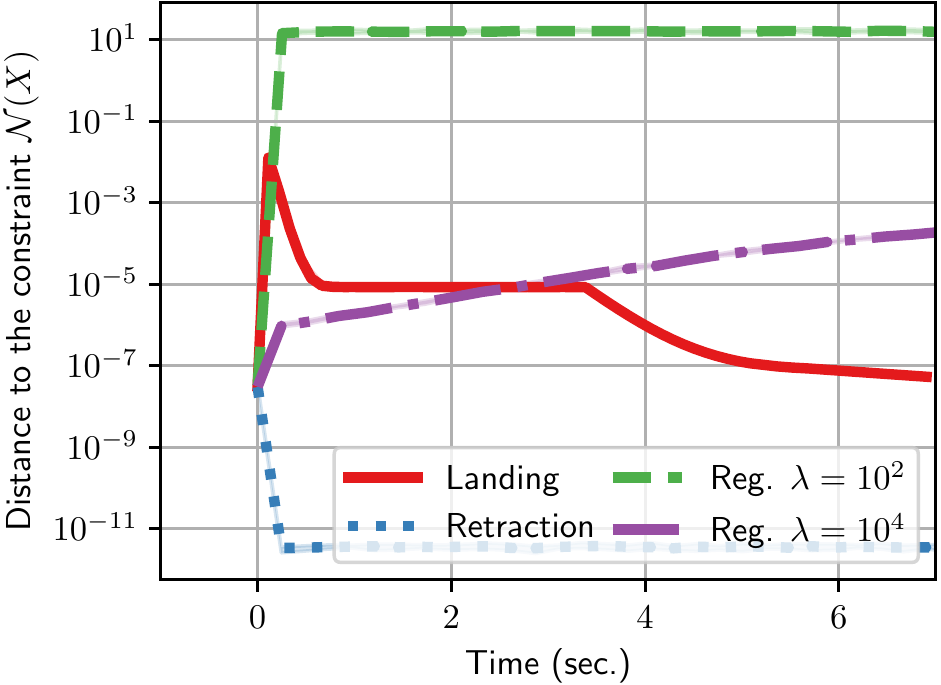}
        \includegraphics[width=.48\textwidth]{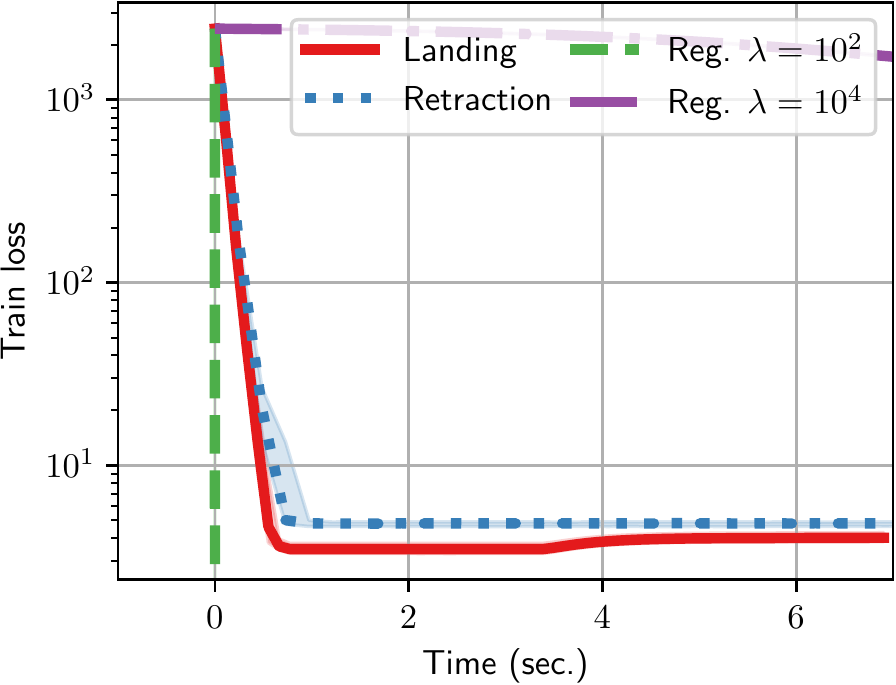}
        \caption{$n=5000, p=100$ \label{fig:pca_p100}}
    \end{subfigure}
    \begin{subfigure}[b]{0.9\textwidth}
        \includegraphics[width=.49\textwidth]{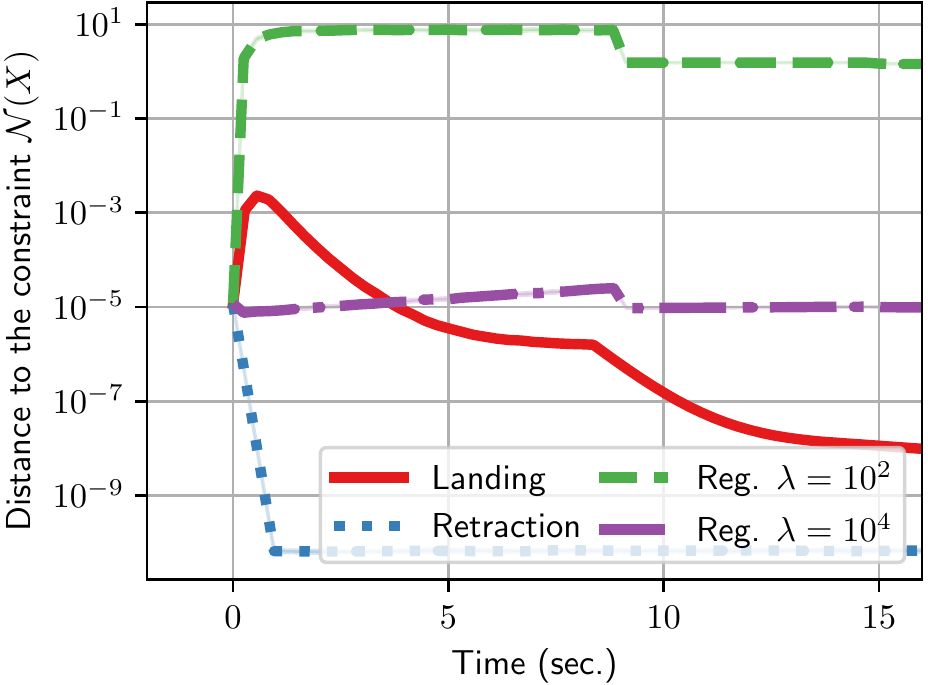}
        \includegraphics[width=.48\textwidth]{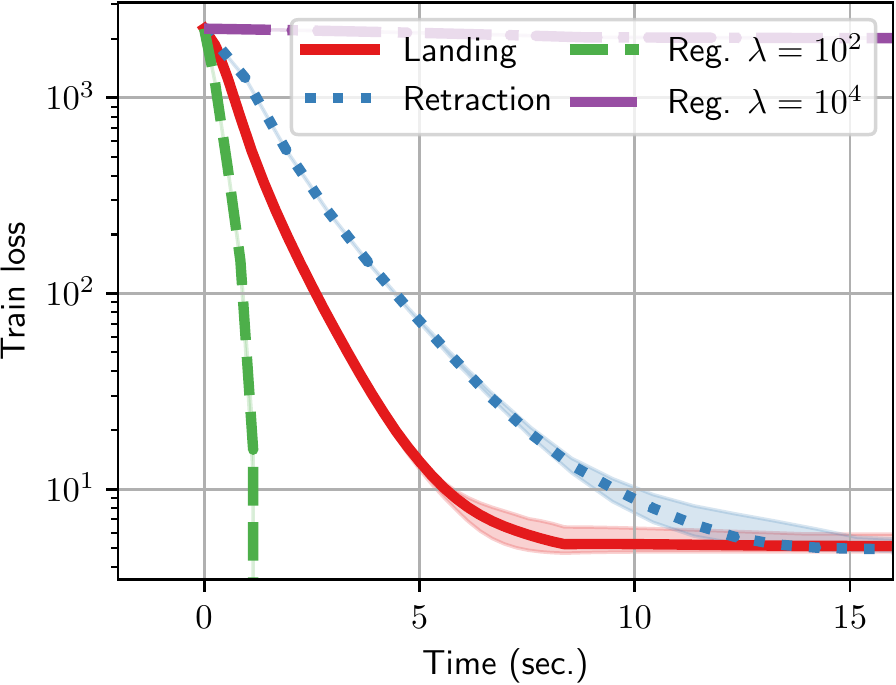}
        \caption{$n=5000, p=500$ \label{fig:pca_p500}}
    \end{subfigure}
    \caption{Convergence plots of the tested methods on online PCA.\label{fig:pca_app}}
\end{figure}

Figure~\ref{fig:ort_conv_app} shows the convergence of training of VGG16 on CIFAR-10 where each convolutional kernel is restricted to be orthogonal with the \texttt{GeoTorch} library using trivializations \citep{lezcano2019trivializations}. We see that the straightforward implementation of the method takes around~10 hours to complete, which is roughly 14 times slower compared to the other methods tested in Section~\ref{sec:algorithms}.
\begin{figure}[h]
    \centering
     \includegraphics[width=.45\textwidth]{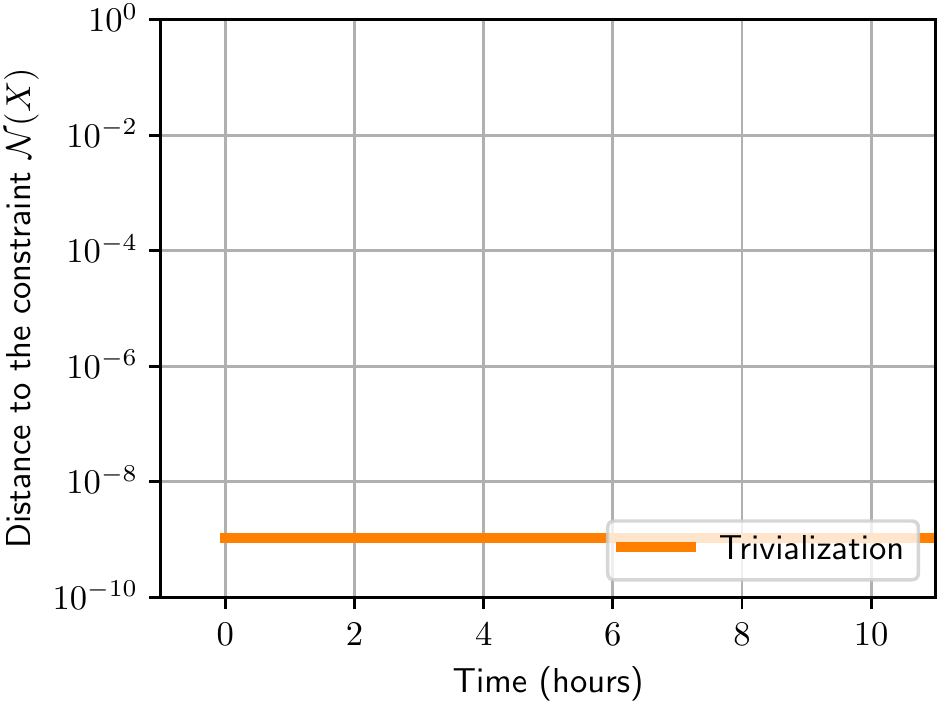}
    \includegraphics[width=.435\textwidth]{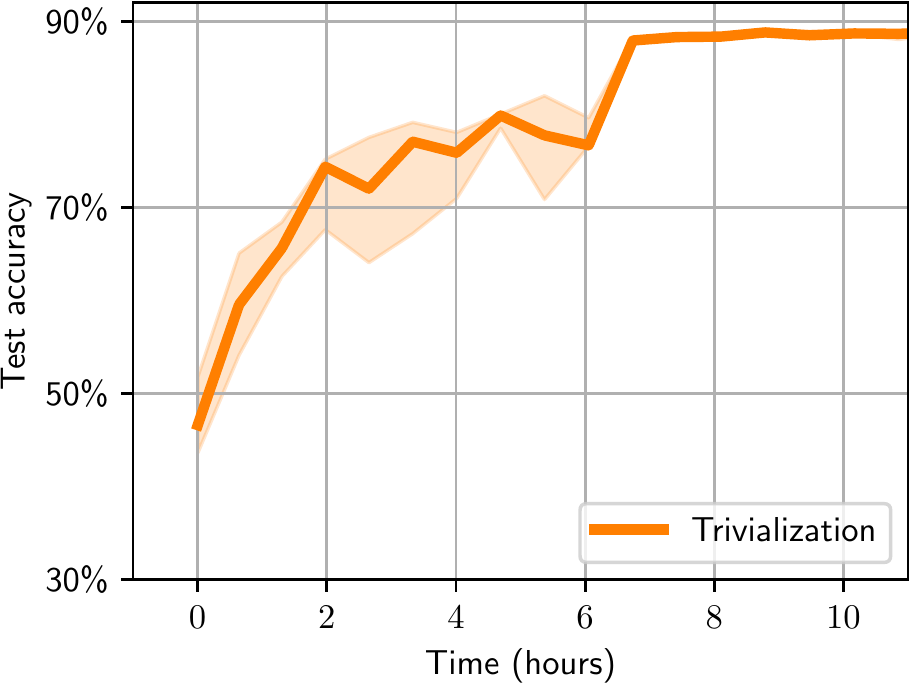}
    \caption{Training of VGG16 on CIFAR10 with orthogonal convolutional kernels with \texttt{GeoTorch} library using trivializations.\label{fig:ort_conv_app}}
\end{figure}

\section{Proofs}
\label{app:proofs}
\subsection{Proof of Lemma~\ref{lemma:singular_values}}
\begin{proof}
Letting $\sigma_1,\dots, \sigma_p$ the singular values of $X$, we have $\|X^{\top}X-I_p\|^2 = \sum_{i=1}^p(\sigma_i^2 - 1)^2$. Hence, if $\|X^{\top}X-I_p\|^2\leq \varepsilon^2$, we must have that all singular values satisfy $|\sigma^2 - 1| \leq \varepsilon$, which proves the result.    
\end{proof}

\subsection{Proof of Proposition~\ref{prop:bound_landing_norm}}
\begin{proof}
  Since the field \eqref{eq:skew_field} is the sum of two orthogonal terms, we have
  $\|F(X, A)\|^2 = \|AX\|^2 + \lambda^2\|X(X^\top X - I_p)\|^2$.
  Using the bound on the singular values of $X$ from Lemma~\ref{lemma:singular_values}, we have $4(1-\varepsilon) \mathcal{N}(X)\leq \|X(X^\top X - I_p)\|^2\leq 4(1+\varepsilon) \mathcal{N}(X)$, 
  which concludes the proof.
\end{proof}

\subsection{Proof of Lemma~\ref{lemma:safe_step}}
\begin{proof}
    Let the residual at $X$ be $\Delta = X^\top X - I_p$ and the landing field be $F = F(X, A)$. The residual $\tilde \Delta = {\tilde X}^\top {\tilde X} - I_p$ at the updated $\tilde X$ can be expressed as
    \begin{align}
        \tilde \Delta &= {\tilde X}^\top {\tilde X} - I_p = (X - \eta F)^\top (X - \eta F) - I_p \\
        &= X^\top X - \eta (X^\top F + F^\top X ) + \eta^2 F^\top F - I_p\\
        &= \Delta - \eta \left(X^\top (A X + \lambda X \Delta) + (-X^\top A + \lambda \Delta X^\top) X \right) + \eta^2 F^\top F \\
        &= \Delta - 2\eta\lambda (\Delta + \Delta^2) + \eta^2 F^{\top} F,
    \end{align}
    \revision{where for the last equality we use that $X^\top X = \Delta + I_p$}. By the triangle inequality and norm submultiplicativity, we can bound the Frobenius norm as
    \begin{align}
        \| \tilde \Delta \| &\leq (1-2\eta\lambda) \| \Delta \| + 2\eta\lambda\| \Delta \|^2 + \eta^2 \| F \|^2,
    \end{align}
    when $\eta < 1/(2\lambda)$. Rewriting the above using scalar notation of $d = \| \Delta \|$ and $g = \| F \|$ and fixed $\lambda >0$, we wish to find an interval for $\eta$ such that
    \begin{equation}
        \eta^2 g^2 + 2\eta\lambda d (d-1) + d \leq \varepsilon, \label{eq:safe_step_quad}
    \end{equation}
    which in turn ensures that $\| \tilde \Delta \| \leq \varepsilon$. For $\varepsilon \geq d$, i.e., when $X$ is in the safe region, we can discard the negative root of the quadratic inequality in \eqref{eq:safe_step_quad} resulting in the upper bound in \eqref{eq:safe_step size}.
    Note that if $d \geq 1$, the upper bound on the step size is not positive when on the boundary $d = \varepsilon$. To prevent this case, we need to have $\varepsilon < 1$.
\end{proof}

\subsection{Proof of Lemma~\ref{lemma:safe_step size_lower}}
\label{app:proof:safe_step}
\begin{proof}
    Let $X$ be with distance $d>0$, i.e.,
    \begin{equation*}
        \| X^\top X - I_p\|^2_F  = \sum_{i=1}^{p} (\sigma_i^2 - 1)^2 = d^2,
    \end{equation*}
    where $\sigma_i$ are the singular values of $X$, which, as a result, must be bounded as: $\sigma_i^2 \in [1-d, 1+d]$.
    Consider the bound on the normalizing component
    \begin{equation*}
        \| \nabla \cN(X) \|_F^2 = \| X (X^\top X - I_p)\|^2_F = \sum_{i=1}^p \sigma_i^2 (\sigma_i^2-1)^2,
    \end{equation*}
    implying that
    \begin{equation*}
        (1-d)d^2 \leq \| \nabla \cN(X) \|_F^2 \leq (1+d)d^2.
    \end{equation*}
    By the orthogonality of the two components in the landing field, we have that
    \begin{equation}
        \| \revision{F(X, A)} \|_F^2 = \|AX\|^2_F + \lambda^2 \| \nabla \cN(X)\|_F^2. \label{eq:safe_step size_flow_bound}
    \end{equation}

    We proceed to lower-bound the first term in the minimum stated in \eqref{eq:safe_step size} to make it independent of $X$. In the following, we denote $a = \|AX\|_F$ and $g = \| F(X, A) \|_F$
    \begin{align}
        &\frac{1}{g^2}\left( \lambda d (1-d) + \sqrt{\lambda^2 d^2 (1-d)^2 + g^2 (\varepsilon - d)}\right) \\
        &\geq \frac{1}{a^2 + \lambda^2 (1+d)d^2}\left(  \lambda d (1-d) + \sqrt{\lambda^2 d^2 (1-d)^2 + (a^2 + \lambda^2(1-d)^2 d^2) (\varepsilon - d)} \right) \label{eq:safe_step size_lower1} \\
        &\geq \frac{1}{a^2 + \lambda^2 (1+d)d^2}\left( \lambda d (1-d)\left(1 + \sqrt{\frac{1+\varepsilon-d}{2}}\right) + a\sqrt{\frac{\varepsilon - d}{2}}\right) \label{eq:safe_step size_lower2}\\
        &\geq \frac{\lambda (1-\varepsilon) d  + a\sqrt{(\varepsilon - d)/2}}{a^2 + \lambda^2 (1+\varepsilon)d^2} \finalrevision{:=K(\lambda, \varepsilon, a, d)} \label{eq:safe_step size_lower3}
    \end{align}
    where in \eqref{eq:safe_step size_lower1} we used the bound from \eqref{eq:safe_step size_flow_bound} to $g$ and that $(1-d)^2d^2 \leq (1-d)d^2$ for $d<1$, in \eqref{eq:safe_step size_lower2} we used the fact that for any $x,y \geq 0$ we have that $\sqrt{x^2 + y^2} \geq (x+y)/\sqrt{2}$, and in \eqref{eq:safe_step size_lower3} we used that $d \leq \varepsilon <1$.

\finalrevision{
We now define 
$$
Q(\lambda, \varepsilon, \tilde{a}) = \inf_{a\in]0, \tilde{a}], d\in ]0, \varepsilon]}K(\lambda, \varepsilon, a, d)\enspace.
$$
We have that for all $\lambda >0, \varepsilon>0$ and $\tilde{a}>0$,  $Q(\lambda,\varepsilon, \tilde{a})>0$ since $K>0$ and $0$ cannot be attained on the boundaries where $a\to 0$ or $d\to 0$. Indeed, as $a\to 0$, we have $K(\lambda, \varepsilon, a, d) = \frac{1-\varepsilon}{\lambda(1+\varepsilon)d} + O(a)$, and as $d\to 0$ we have $K(\lambda, \varepsilon, a, d) = \frac{\sqrt{\varepsilon}}{a\sqrt{2}} + O(d)$, both of these limits are bounded away from 0.
} 
    As a result, the upper bound on the safeguard step size for all $X$ with the $\varepsilon$ distance cannot decrease below
    \begin{equation*}
        \eta(X) \geq \min \eta^*(\tilde{a}, \varepsilon, \lambda) := \left\{Q(\lambda, \varepsilon, \tilde{a}),\, \frac{1}{2\lambda}\right\}
    \end{equation*}
\end{proof}

\subsection{Proof of Proposition~\ref{prop:recursive_step_size}}
\begin{proof}
    The proof is a simple recursion. The property that $X_0\in\stiefel^\varepsilon$ holds by assumption. Then, assuming that $X_0, \dots, X_k\in\stiefel^\varepsilon$, we have by assumption that $A_k$ is such that $\|A_kX_k\|\leq \tilde{a}$. It follows that $X_{k+1}\in\stiefel^\varepsilon$ following Lemma~\ref{lemma:safe_step size_lower}, which concludes the recursion.
\end{proof}

\begin{lemma}[Jacobian of $\Phi(X)$] \label{lemma:jacobian}
    Let $\Phi(X) = \sym(\nabla f(X)^{\top} X) :\bR^{n\times p} \rightarrow \bR^{p\times p}$ , 
    let $\Jacob{\Phi}{X}$ denote its derivative at $X$, and let $\Jacob{\Phi}{X}^*[\dot{X}]$ denote its adjoint in the sense of the Frobenius inner product at $X$ applied to $\dot{X}$. Let $\vect{\cdot}:\bR^{m \times n} \rightarrow \bR^{mn}$ denote the vectorization operation, and let $\vectHess{X}$ denote the matrix representation of the Hessian of $f$ at $X$; namely, $\vectHess{X} \in \mathbb{R}^{np\times np}$ such that, for all $\dot{X}\in\mathbb{R}^{n\times p}$, $\vectHess{X}[\vect{\dot{X}}] = \vect{\mathrm{D}(\nabla f)(X)[\dot{X}]}$. Then 
    \begin{equation*}
        \vect{\Jacob{\Phi}{X}^* [X^{\top}X - I_p]} = \vectHess{X} \vect{\nabla \cN (X)} + \vect{\nabla f(X)(X^{\top}X - I)}.
    \end{equation*}
\end{lemma}
\begin{proof}
    Let $\mathrm{D}F(X)[\dot{X}] = \lim_{t\to0} (F(X+t\dot{X})-F(X))/t$ denote the derivative of $F$ at $X$ along $\dot{X}$. We have $\mathrm{D}\Phi(X)[\dot{X}] = \sym(\left(\mathrm{D}(\nabla f)(X)[\dot{X}]\right)^{\top} X + \nabla f(X)^{\top} \dot{X})$. Hence, for all $\dot{X}\in\mathbb{R}^{n\times p}$ and $Z\in \mathbb{R}^{p\times p}$,
    it holds that 
    \begin{align}
        \langle \dot{X}, (\mathrm{D}\Phi)^*(X)[Z]\rangle 
        &= \inner{\mathrm{D}\Phi(X)[\dot{X}]}{Z} \\
        &= \inner{\sym(\left(\mathrm{D}(\nabla f)(X)[\dot{X}]\right)^{\top} X + \nabla f(X)^{\top} \dot{X})}{Z}\\
        &= \inner{\left(\mathrm{D}(\nabla f)(X)[\dot{X}]\right)^{\top} X + \nabla f(X)^{\top} \dot{X}}{\sym(Z)}\\
        &=\inner{\mathrm{D}(\nabla f)(X)[\dot{X}]}{ X\sym(Z)} + \inner{ \dot{X}}{\nabla f(X)\sym(Z)}\\
        &= \inner{\dot{X}}{\mathrm{D}(\nabla f)(X)[X\sym(Z)]} + \inner{\dot{X}}{\nabla f(X)\sym(Z)}.
    \end{align}
    Hence 
    \[
    (\mathrm{D}\Phi)^*(X)[Z] = \mathrm{D}(\nabla f)(X)[X\sym(Z)] + \nabla f(X)\sym(Z),
    \]
    where $(\mathrm{D}\Phi)^*(X)$ is the adjoint of the Jacobian of $\Phi$ at $X$, and $\mathrm{D}(\nabla f)(X)[\cdot]: \mathbb{R}^{n\times p} \to \mathbb{R}^{n\times p}$ is the Hessian operator of $f$ at $X$. For $Z = X^{\top} X-I_p$, this yields the result of Lemma~\ref{lemma:jacobian}.
\end{proof}

\subsection{Proof of Proposition~\ref{prop:lyapunov}}
\begin{proof}
    Computing the gradient of $\cL(X)$ has four terms
    \begin{equation}\label{eq:merit_grad}
        \nabla \cL(X) = \nabla f(X) -\frac12 \Jacob{\Phi}{X}^* [X^{\top}X - I_p] - X\sym(X^\top \nabla f(X)) + \mu \nabla \cN(X),
    \end{equation}
    where $\Jacob{\Phi}{X}^* [X^{\top}X - I_p]$ denotes the adjoint of the Jacobian in the sense of the Frobenius inner product of $\Phi(X) = \sym(X^\top \nabla f(X))$ in $X$ evaluated in the direction $X^\top X - I_p$. We proceed by expressing the inner product of the landing field $\Lambda(X)$ with each of the four terms of $\nabla \cL(X)$ in \eqref{eq:merit_grad} separately.
    
    The inner product between the first term of \eqref{eq:merit_grad} and the landing field $\Lambda(X)$ is
    \begin{align}
        \inner{\Lambda(X)}{\nabla f(X)} &= \inner{\sk(\nabla f(X) X^\top)X + \lambda X (X^\top X - I_p)}{\nabla f(X)}\\
        &= \inner{\sk(\nabla f(X) X^\top)}{\nabla f(X)X^\top} + \lambda\inner{X^\top X - I_p}{X^\top \nabla f(X)}\\
        & = \| \sk(\nabla f(X) X^\top) \|_F^2 + \lambda \inner{\sym(X^\top \nabla f(X))}{X^{\top}X -I_p}, \label{eq:merit_inner1}
    \end{align}
    where we used the fact that the inner product of a skew-symmetric and a symmetric matrix is zero.

    We can express the inner product between the Jacobian using Lemma \ref{lemma:jacobian} in the second term of \eqref{eq:merit_grad} as 
    \begin{align}
        \begin{split}
        \inner{\Lambda(X)}{-\frac12 \Jacob{\Phi}{X}^* [X^{\top}X - I_p]} = &-\frac12\lambda           \inner{\vectHess{X} \vect{ \nabla \cN(X)}}{ \vect{\nabla \cN(X)}} \\
            \qquad & -\frac12\inner{\vectHess{X} \vect{ \nabla \cN(X)}}{\vect{\grad f(X)}}\\
            \qquad & -\frac12\inner{\nfX^{\top} \grad f(X)}{X^{\top}X - I_p} \\
            \qquad & -\frac12\lambda \inner{\sym(X^\top \nabla f(X))}{(X^{\top}X - I_p)^2}, 
        \end{split}\label{eq:merit_inner2}
    \end{align}
    where $\vect{\cdot}:\mathbb{R}^{n\times p} \rightarrow \mathbb{R}^{np}$ vectorizes a matrix and $\vectHess{X}$ denote the matrix representation of the Hessian of $f$ at $X$; namely, $\vectHess{X} \in \mathbb{R}^{np\times np}$ such that, for all $\dot{X}\in\mathbb{R}^{n\times p}$, $\vectHess{X}[\vect{\dot{X}}] = \vect{\mathrm{D}(\nabla f)(X)[\dot{X}]}$.
    
    The third term is
    \begin{align}
        \inner{\Lambda(X)}{-X\sym(X^\top \nabla f(X))} &= \inner{\sk(\nabla f(X) X^\top)X + \lambda \nabla \cN(X)}{-X\sym(X^\top \nabla f(X))}\\
        &= -\lambda\inner{\sym(X^\top \nabla f(X))}{X^\top X (X^\top X - I_p)}, \label{eq:merit_inner3}
    \end{align} 
    where the inner product with the skew-symmetric matrix in the first term is zero.
    
    The inner product of the fourth term and the landing field is
    \begin{align}
        \inner{\Lambda(X)}{\mu \nabla \cN(X)} &= \lambda \mu \| \nabla \cN(X) \|^2_F \label{eq:merit_inner4}
    \end{align}
    by orthogonality of the two components of the landing field.

    Adding all the four terms expressed in \eqref{eq:merit_inner1}, \eqref{eq:merit_inner2},  \eqref{eq:merit_inner3}, and  \eqref{eq:merit_inner4} together gives
    \begin{align}
         \inner{\nabla \cL(X)}{\Lambda(X)} &= \| \sk(\nabla f(X) X^\top) \|_F^2 \label{eq:merit_innerB0}\\
            & + \lambda \inner{\left(\mu I_{np} - \frac12 \vectHess{X}\right) \vect{\nabla\cN(X)}}{\vect{\nabla\cN(X)}} \label{eq:merit_innerB1}\\
            & - \lambda \frac32 \inner{(X^{\top}X-I_p)^2}{\sym(X^\top \nabla f(X))} \label{eq:merit_innerB2}\\
            & -\frac12\inner{\vectHess{X} \vect{ \nabla \cN(X)}}{\vect{\grad f(X)}} \label{eq:merit_innerB3}\\
            & -\frac12\inner{\nfX^{\top} \grad f(X)}{X^{\top}X - I_p}, \label{eq:merit_innerB4}
    \end{align}
    where the first line in \eqref{eq:merit_innerB0} comes from the first term of \eqref{eq:merit_inner1}, the second line \eqref{eq:merit_innerB1} is a combination of the first term in \eqref{eq:merit_inner2} and \eqref{eq:merit_inner4}, the third line in \eqref{eq:merit_innerB2} comes from the second term in \eqref{eq:merit_inner1}, the third term in \eqref{eq:merit_inner2},  and \eqref{eq:merit_inner2}.

    We will bound lower bound each term separately. The term in \eqref{eq:merit_innerB1} is lower bounded as
    \begin{align}
        \lambda \inner{\left(\mu I_{np} - \frac12 \vectHess{X}\right) \vect{\nabla\cN(X)}}{\vect{\nabla\cN(X)}} &\geq \lambda \left(\mu-\frac{\Lc}{2}\right)\| \nabla\cN(X) \|^2_F \\
            &\geq 4 \lambda \left(\mu-\frac{\Lc}{2}\right) \sigma_p^2 \, \cN(X), \label{eq:merit_boundB1} 
    \end{align}
    where $\sigma_p$ is the smallest singular value of $X$ and $\Lc$ is the Lipschitz constant of $\nabla f(X)$.

    The term in \eqref{eq:merit_innerB2} is lower bounded using the Cauchy-Schwarz inequality as
    \begin{align}
        - \lambda \frac32 \inner{(X^{\top}X-I_p)^2}{\sym(X^\top \nabla f(X))} \geq - 6 \lambda \, \cN(X) \| \sym(X^\top \nabla f(X)) \|_F , \label{eq:merit_boundB2}
    \end{align}
    by the fact that $\|X^\top X - I_p\|^2_F = 4\cN(X)$.
    
    The term in \eqref{eq:merit_innerB3} is lower bounded as
    \begin{align}
         -\frac12\inner{\vectHess{X} \vect{ \nabla \cN(X)}}{\vect{\grad f(X)}} &\geq -\frac{\Lc}{2} \|  X (X^\top X - I_p) \|_F \|\grad f(X)\|_F \\ 
         & \geq -\Lc \sigma_1 \sqrt{\cN(X)} \|\grad f(X)\|_F, \label{eq:merit_boundB3}
    \end{align} 
    by the fact that the operator norm of $\|X\|_2 = \sigma_1 $ and $\| X^\top X - I_p\|_F = 2\sqrt{\cN(X)}$.
    
    The term in \eqref{eq:merit_innerB4} is lower bounded as
    \begin{align}
        -\frac12\inner{\nfX^{\top} \grad f(X)}{X^{\top}X - I_p} & \geq -\frac12\|\nabla f(X)\|_F \| \grad f(X) \|_F \| X^\top X - I_p \|_F\\
        &\geq - \gradc \sqrt{\cN(X)} \|\grad f(X)\|_F,\label{eq:merit_boundB4}
    \end{align}
    where $\gradc$ is such that $\|\nabla f(X)\| \leq \gradc$ for all $X\in\stiefel^\varepsilon$.

    Now we bound the terms in \eqref{eq:merit_boundB3} and \eqref{eq:merit_boundB4} together
    \begin{align}
         -\frac12\inner{\vectHess{X} \vect{ \nabla \cN(X)}}{\vect{\grad f(X)}}  &-\frac12\inner{\nfX^{\top} \grad f(X)}{X^{\top}X - I_p} \nonumber\\
         & \geq - (\gradc + \Lc\sigma_1)  \sqrt{\cN(X)} \| \grad f(X) \|  \\ 
         & \geq - \frac12 (\gradc + \Lc\sigma_1) \left( \beta \cN(X) + \beta^{-1} \| \grad f(X)\|^2 \right)    \label{eq:merit_boundB3and4},
    \end{align}
    where in the first inequality we used the previously derived bounds in \eqref{eq:merit_boundB3} and \eqref{eq:merit_boundB4}, and the second inequality is a consequence of the AG inequality $\sqrt{xy}\leq (x+y)/2$ with $x = \beta \cN(X)$ and $y = \beta^{-1} \| \grad f(X)\|^2$ for an arbitrary $\beta >0$, which will be specified later.
    
    Adding the bounds in \eqref{eq:merit_boundB1}, \eqref{eq:merit_boundB2}, and \eqref{eq:merit_boundB3and4} together, we have a lower bound as
    \begin{align}
        \inner{\nabla \cL(X)}{\Lambda(X)} &\geq \|\grad f(X)\|_F^2\left( \sigma_1^{-2} - \frac12(\gradc + \Lc\sigma_1) \beta^{-1} \right) \\
            &+\cN(X) \left(  4\lambda \left(\mu-\frac{\Lc}{2}\right) \sigma_p^2 - 6 \lambda \| \sym(X^\top \nabla f(X)) \|_F  - \frac12(\gradc + \Lc\sigma_1) \beta \right)
    \end{align}
    where we used that $\| \sk(\nabla f(X) X^{\top})\|_F \geq \sigma_1^{-1} \| \grad f(X)\|_F$. Choosing $\beta = \sigma_1^2 \frac{\gradc+\Lc\sigma_1}{2-\sigma_1^2}$ and bounding $\| \sym(X^\top \nabla f(X)) \|_F \leq s$ we have
    \begin{align}
         \inner{\nabla \cL(X)}{\Lambda(X)} &\geq \frac12 \|\grad f(X)\|_F^2 + \cN(X) \left( 4 \lambda \left(\mu-\frac{\Lc}{2}\right) \sigma_p^2 - 6 \lambda  s  - \frac12\sigma_1^2 \frac{(\gradc+\Lc\sigma_1)^2}{2-\sigma_1^2} \right)\\
         &\geq \frac12 \|\grad f(X)\|_F^2 + \cN(X) \left( 4 \lambda \left(\mu-\frac{\Lc}{2}\right) (1-\varepsilon) - 6 \lambda  s  - \frac12 \hat{L}^2(1+\varepsilon) \frac{(2\sqrt{1+\varepsilon})^2}{2-(1+\varepsilon)} \right) \\
         &\geq \frac12 \|\grad f(X)\|_F^2 + \cN(X) \left( 4 \lambda \left(\mu-\frac{\Lc}{2}\right) (1-\varepsilon) - 6 \lambda  s  - \finalrevision{2 \hat{L}^2 \frac{(1+\varepsilon)^2}{1-\varepsilon}} \right), \label{eq:merit_bound_final1}
    \end{align}
    where in the second line we define $\hat{L} = \max\{\Lc, \gradc\}$ and we used that $\sqrt{1+\varepsilon} \geq 1$. 
    The coefficient in front of the distance term $\cN(X)$ in \eqref{eq:merit_bound_final1} is lower bounded by $\lambda \mu$ for
    \begin{equation}  \label{eq:mu-lb}
        \mu \geq \frac{2}{3-4\varepsilon}\left( \Lc(1-\varepsilon) + 3s + \finalrevision{\hat{L}^2 \frac{(1+\varepsilon)^2}{\lambda(1-\varepsilon)}} \right).
    \end{equation}
\end{proof}

\subsection{Proof of Proposition~\ref{prop:bound_scalar_with_norm}}
\begin{proof}
    We have
    \begin{align}
         \frac12\|\grad f(X)\|^2 + \nu \mathcal{N}(X)&\geq \rho(\|\grad f(X)\|^2 + 4\lambda^2 (1+\varepsilon) \mathcal{N}(X))\\
         &\geq \rho \|\Lambda(X)\|^2\enspace,
    \end{align}
where the last inequality comes from Proposition~\ref{prop:bound_landing_norm}.
\end{proof}

\subsection{Proof of Proposition~\ref{prop:smoothness_merit}}
\begin{proof}
    We start by considering the smoothness constant of $\mathcal{N}$. We have $\nabla \mathcal{N}(X) = X(X^{\top}X-I_p)$, hence we find that its Hessian is such that in a direction $E\in\mathbb{R}^{n \times p}$:
    $$
    \nabla^2\mathcal{N}(X)[E] = E(X^{\top}X - I_p) +X(E^{\top}X + X^{\top}E)\enspace.
    $$
    We can then bound crudely using the triangular inequality
    \begin{align}
        \|\nabla^2\mathcal{N}(X)[E]\| &\leq \|E(X^{\top}X - I_p)\| +\|XE^{\top}X\| + \|XX^{\top}E\|\\
        &\leq \varepsilon\|E\| + 2(1+\varepsilon)\|E\|\\
        &\leq (2 + 3\varepsilon)\|E\|\enspace.
    \end{align}
    This implies that all the absolute values of the eigenvalues of $\nabla^2\mathcal{N}(X)$ are bounded by $2+3\varepsilon$, so that $\mathcal{N}$ is $(2+3\varepsilon)$-smooth.
    The result follows since the sum of smooth functions is smooth.
\end{proof}

\subsection{Proof of Proposition~\ref{prop:convergence_landing}}
\begin{proof}
Since we use the safeguard step size, the iterates remain in $\stiefel^\varepsilon$.
  Using the $L_g$-smoothness of $\mathcal{L}$ (Proposition~\ref{prop:smoothness_merit}), we get
  $$
  \mathcal{L}(X_{k+1})\leq \mathcal{L}(X_k) - \eta \langle \Lambda(X_k), \nabla \mathcal{L}(X_k)\rangle + \frac{L_g\eta^2}2\|\Lambda(X_k)\|^2 \enspace.
  $$
  Using Proposition~\ref{prop:lyapunov} and Proposition~\ref{prop:bound_landing_norm}, we get
  $$
  \mathcal{L}(X_{k+1}) \leq \mathcal{L}(X_k) - (\frac\eta2 - \frac{\eta^2L_g}2)\|\grad f(X_k)\|^2 -(\eta\nu -2\lambda^2\eta^2L_g(1+\varepsilon))\mathcal{N}(X_k)\enspace.
  $$
  Using the hypothesis on $\eta$, we have both $\frac\eta2 - \frac{\eta^2L_g}2 \geq \frac\eta4$ and $\eta\nu -2\lambda^2\eta^2L_g(1+\varepsilon) \geq \frac{\eta\nu}2$. We, therefore, obtain the inequality
  
  $$
  \frac{\eta}4\|\grad f(X_k)\|^2 + \frac{\eta\nu}2 \mathcal{N}(X_k) \leq \mathcal{L}(X_k) - \mathcal{L}(X_{k+1})\enspace.
  $$
  
  Summing these terms gives
  \begin{equation}
       \frac{\eta}4\sum_{k=1}^K\|\grad f(X_k)\|^2 + \frac{\eta\nu}2 \sum_{k=1}^K\mathcal{N}(X_k) \leq \mathcal{L}(X_0) - \mathcal{L}(X_{K+1}) \leq \mathcal{L}(X_0) - \mathcal{L}^*,
  \end{equation}
which implies that we have both
\begin{equation}
    \frac{\eta}4\sum_{k=1}^K\|\grad f(X_k)\|^2 \leq \mathcal{L}(X_0) - \mathcal{L}^*\enspace \text{and} \enspace
    \frac{\eta\nu}2 \sum_{k=1}^K\mathcal{N}(X_k)  \leq \mathcal{L}(X_0) - \mathcal{L}^*\enspace.
\end{equation}
These two inequalities then directly provide the result.
\end{proof}

\subsection{Proof of Proposition~\ref{prop:decrease_stochastic}}
\begin{proof}We use once again the smoothness of $\mathcal{L}$ and unbiasedness of $\Lambda_i(X_k)$ to get

\begin{align}
    \mathbb{E}_i[\mathcal{L}(X_{k+1})]&\leq\mathcal{L}(X_k) - \eta_k\langle \Lambda(X_k), \nabla\mathcal{L}(X_k)\rangle+\frac{\eta_k^2L_g}2\mathbb{E}[\|\Lambda_i(X_k)\|^2]\\
    &\leq \mathcal{L}(X_k) - \eta_k\langle \Lambda(X_k), \nabla\mathcal{L}(X_k)\rangle+\frac{\eta_k^2L_g}2(B + \|\Lambda(X_k)\|^2) \\
    &\leq \mathcal{L}(X_k) - \eta_k(\frac12\|\grad f(X_k)\|^2 + \nu \mathcal{N}(X_k))\\
    &+ \frac{\eta_k^2L_g}2(\|\grad f(X_k)\|^2 + 4\lambda^2(1+\varepsilon)\mathcal{N}(X_k))+ \frac{\eta_k^2L_gB}2 \\
    &\leq  \mathcal{L}(X_k) - \frac{\eta_k - \eta_k^2L_g}2\|\grad f(X_k)\|^2 - (\eta_k\nu - 2\eta_k^2L_g\lambda^2(1+\varepsilon))\mathcal{N}(X_k )+\frac{\eta_k^2L_gB}2\enspace.
\end{align}
Taking $\eta_k\leq \min(\frac1{2L_g}, \frac{\nu}{4L_g\lambda^2(1+\varepsilon)})$ simplifies the inequality to 
$$
\mathbb{E}_i[\mathcal{L}(X_{k+1})]\leq \mathcal{L}(X_k) - \frac{\eta_k}4 \|\grad f(X_k)\|^2 - \frac{\eta_k\nu}2\mathcal{N}(X_k)+\frac{\eta_k^2L_gB}2\enspace.
$$
\end{proof}
\subsection{Proof of Proposition~\ref{prop:convergence_stochastic_decreasing_step}}
\begin{proof}
Taking expectations with respect to the past, and summing up the inequality in Proposition~\ref{prop:decrease_stochastic} gives, using $\sum_{k=0}^K (1 + k)^{-1}\leq \log(K +1)$:
$$
\frac14\sum_{k=0}^K\eta_k\mathbb{E}[\|\grad f(X_k)\|^2] \leq \mathcal{L}(X_0) - \mathcal{L}^* + \frac{\eta_0^2L_g B}2 \log(K+1)
$$
and 
$$
\frac\nu2\sum_{k=0}^K\eta_k\mathbb{E}[\mathcal{N}(X_k)] \leq \mathcal{L}(X_0) - \mathcal{L}^* + \frac{\eta_0^2L_g B}2 \log(K+1)\enspace.
$$

Next, we use the bound $\inf_{k\leq K}\mathbb{E}[\|\grad f(X_k)\|^2]\leq \sum_{k=0}^K\eta_k\|\grad(X_k)\|^2 \times (\sum_{k=0}^K\eta_k)^{-1}$ and the fact that $\sum_{k=0}^K\eta_k\geq \eta_0\sqrt{K}$
to get
$$
\inf_{k\leq K}\mathbb{E}[\|\grad f(X_k)\|^2] \leq 4\frac{ \mathcal{L}(X_0) - \mathcal{L}^* + \frac{\eta_0^2L_g B}2 \log(K+1)}{\eta_0\sqrt{K}}
$$
and 
$$
\inf_{k\leq K}\mathbb{E}[\mathcal{N}(X_k)] \leq 2\frac{ \mathcal{L}(X_0) - \mathcal{L}^* + \frac{\eta_0^2L_g B}2 \log(K+1)}{\nu\eta_0\sqrt{K}} \enspace.
$$
\end{proof}

\subsection{Proof of Proposition~\ref{prop:convergence_stochastic_fixed_step}}
\begin{proof}
    Just like for the previous proposition, we get by summing the descent lemmas:

    $$\frac{1}{4K}\sum_{k=1}^K\mathbb{E}[\|\grad f(X_k)\|^2] + \frac{\nu}{2K} \sum_{k=1}^K\mathbb{E}[\mathcal{N}(X_k)]\leq \frac{\mathcal{L}(X_0) - \mathcal{L}^*}{\eta K} + \frac{\eta L_gB}2\enspace .$$

Taking $\eta = \eta_0K^{-1/2}$ yields as advertised:
$$
\inf_{k\leq K}\mathbb{E}[\|\grad f(X_k)\|^2] \leq \frac{4}{\sqrt{K}}\left( \frac{\mathcal{L}(X_0) - \mathcal{L}^*}{\eta_0} + \frac{\eta_0 L_gB}2\right)
$$
and 
$$
\inf_{k\leq K}\mathbb{E}[\mathcal{N}(X_k)] \leq \frac{2}{\nu \sqrt{K}}\left( \frac{\mathcal{L}(X_0) - \mathcal{L}^*}{\eta_0} + \frac{\eta_0 L_gB}2\right)\enspace.
$$
\end{proof}
\subsection{Proof of Proposition~\ref{prop:convergence_saga}}
\begin{proof}
    We define $X_k^i$ as the matrix such that $\Phi_k^i = \nabla f(X_k^i)$, i.e., the last iterate for which the memory corresponding to sample $i$ has been updated. As in the classical SAGA analysis, we will form a merit function combining the merit function $\mathcal{L}$ and the distance to the memory, defined as 
    $$S_k = \frac1N\sum_{j=1}^N\mathbb{E}[\|X_k - X_k^j\|^2]\enspace.$$
We also define the variance of the landing direction as 
$$
V_k = \frac1N\sum_{j=1}^N\|\Lambda_k^j\|^2\enspace.
$$
For short, we let $\Lambda_k = \Lambda(X_k)$.
    \paragraph{Control of the distance to the memory}
    Looking at an individual term $j$ in the sum of $S_k$, we have
    \begin{align}
        \mathbb{E}_i[\|X_{k+1} - X_{k+1}^j\|^2] &= \mathbb{E}_i[\|X_{k} - \eta \Lambda_k^i -  X_{k+1}^j\|^2]\\
        &= \frac1N (\sum_{i\neq j}^N\|X_{k} - \eta \Lambda_k^i -  X_{k}^j\|^2 + \eta^2 \|\Lambda_k^j\|^2)\\
        &=\mathbb{E}_i[\|X_{k} - \eta \Lambda_k^i -  X_{k}^j\|^2] -\frac1N(\|X_{k} - \eta \Lambda_k^j -  X_{k}^j\|^2 -\eta^2 \|\Lambda_k^j\|^2)\enspace.
    \end{align}
    On the one hand, we have for the first term:
    \begin{align}
        \mathbb{E}_i[\|X_{k} - \eta \Lambda_k^i -  X_k^j\|^2] &= \|X_{k}  -  X_k^j\|^2 -2\eta \langle \Lambda_k, X_{k}  -  X_k^j\rangle +\eta^2V_k\\
        &\leq (1+\eta \beta)\|X_{k}  -  X_k^j\|^2 + \eta \beta^{-1}\|\Lambda_k\|^2 + \eta^2V_k,
    \end{align}
    where we introduce $\beta >0$ from Young's inequality to control the scalar product.
    For the second term, we obtain
    \begin{align}
        \|X_{k} - \eta \Lambda_k^j -  X_{k}^j\|^2 - \eta^2\|\Lambda_k^j\|^2 &= \|X_k - X_k^j\|^2 - 2\eta \langle  \Lambda_k^j, X_k - X_k^j\rangle\\
        &\geq (1 - \gamma \eta)\|X_k - X_k^j\|^2 - \gamma^{-1} \eta\|\Lambda_k^j\|^2,
    \end{align}
    where once again we introduce $\gamma>0$ from Young's inequality.
    Taking all of these inequalities together gives
    $$\mathbb{E}_i[\|X_{k+1} - X_{k+1}^j\|^2] \leq (1 - \frac1N +\eta\beta + N^{-1}\gamma \eta) \|X_k - X_k^j\|^2 + \eta \beta^{-1}\|\Lambda_k\|^2 + \eta^2V_k + N^{-1}\gamma^{-1}\eta \|\Lambda_k^j\|^2,
    $$
    and averaging these for $j=1\dots N$ gives
    \begin{equation}
        S_{k+1}\leq (1 - \frac1N +\eta\beta + N^{-1}\gamma \eta)S_k + \eta\beta^{-1}\|\Lambda_k\|^2 +(\eta^2 +N^{-1}\gamma^{-1}\eta)V_k \enspace.
    \end{equation}
    We choose $\beta = (4N\eta)^{-1}$ and $\gamma=(4\eta)^{-1}$, which finally gives, assuming $N>4$:
    \begin{equation}
    \label{eq:distance_memory_saga}
    S_{k+1}\leq (1 - \frac1{2N})S_k + 4N\eta^2\|\Lambda_k\|^2 +2\eta^2 V_k \enspace.
    \end{equation}

    \paragraph{Control of the merit function}
    The smoothness of the merit function gives once again
    \begin{align}
        \mathbb{E}_i[\mathcal{L}(X_{k+1})]&\leq \mathcal{L}(X_{k}) - \eta \langle \Lambda_k, \nabla \mathcal{L}(X_k)\rangle + \eta^2\frac{L_g}2V_k\\
        &\leq \mathcal{L}(X_{k}) - \eta \rho \| \Lambda_k\|^2+ \eta^2\frac{L_g}2V_k,
    \end{align}
    where the last inequality comes from Proposition~\ref{prop:bound_scalar_with_norm}.
    \paragraph{Variance control}
    We then control $V_k$. By the bias-variance decomposition, using the unbiasedness of $\Lambda_k^j$, we get
    \begin{align}
    V_k &= \|\Lambda_k\|^2 +\frac1N\sum_{j=1}^N\|\Lambda_k^j - \Lambda_k\|^2\\
    &\leq \|\Lambda_k\|^2 +\frac1N\sum_{j=1}^N\|\grad f_i(X_k) - \sk(\Phi_iX_k^\top)X_k\|^2\\
    &\leq\|\Lambda_k\|^2 +\frac1N\sum_{j=1}^N\|\sk((\nabla f(X_k) - \nabla f(X_k^i))X_k^\top)X_k\|^2\\
    &\leq \|\Lambda_k\|^2 + L^2_f(1+\varepsilon)S_k,
    \end{align}
    where $L_f$ is the smoothness constant of $f$.
    \paragraph{Putting it together}
    We get the two inequalities
    \begin{align}
            S_{k+1}&\leq (1 - \frac1{2N} + 2\eta^2L_f^2(1+\varepsilon)^2)S_k + (4N + 2)\eta^2\|\Lambda_k\|^2\\
            \mathbb{E}_i[\mathcal{L}(X_{k+1})]&\leq \mathcal{L}(X_{k}) - (\eta \rho -\eta^2\frac{L_g}2) \| \Lambda_k\|^2+ \eta^2\frac{L_g}2L_f^2(1+\varepsilon)^2S_k \enspace .
    \end{align}
    The hypothesis on the step size simplifies these inequalities to 

    \begin{align}  
S_{k+1}&\leq (1 - \frac1{4N})S_k + (4N + 2)\eta^2\|\Lambda_k\|^2\\
\mathbb{E}_i[\mathcal{L}(X_{k+1})]&\leq \mathcal{L}(X_{k}) - \frac12\eta \rho \| \Lambda_k\|^2+ \eta^2\frac{L_g}2L_f^2(1+\varepsilon)^2S_k\enspace .
    \end{align}
    We now look for a decreasing quantity of the form  $\mathcal{G}_k = \mathcal{L}(X_{k}) + cS_k$, and get
    $$
    \mathbb{E}[G_{k+1}] \leq G_k -(\frac12\eta\rho -c(4N + 2)\eta^2)\|\Lambda_k\|^2 -(\frac c{4N} - \eta^2\frac{L_g}2L_f^2(1+\varepsilon)^2)S_k\enspace.
    $$

    We take $c = 2N \eta^2L_gL_f^2(1+\varepsilon)^2$ in order to cancel the last term. 
    The term in front of $\|\Lambda_k\|^2$ becomes $\frac12\eta\rho - 2N(4N+2)L_gL_f^2(1+\varepsilon)^2\eta^4$, which is lower bounded by $\frac14\eta\rho$ with the last condition on the step size.
    We therefore get
    $$
    \mathbb{E}[G_{k+1}] \leq G_k -\frac14\eta\rho\|\Lambda_k\|^2\enspace.
    $$
    Taking expectations with respect to the whole past and summing yields
    $$
    \frac1K\sum_{k=1}^K\mathbb{E}[\|\Lambda_k\|^2]\leq \frac4{\eta\rho K}(\mathcal{L}(X_0) - \mathcal{L}^*),
    $$
    which concludes the proof.
\end{proof}

\subsection{Proof of Proposition~\ref{prop:linear_problem_with_constraint}}
\begin{proof}
    Let us write $X = QDW^\top$ a singular value decomposition of $X$. Then we find 
    $$g(X) = \langle U^\top QDW^{\top}V, \Sigma\rangle + \frac{\lambda}4\|D^2 - I_p\|^2$$
    The factors $Q$ and $W$ appear only in the first term. Hence we will first consider its minimization. 
    Letting $\tilde{U} = Q^{\top}U$ and $\tilde{V} =W^{\top}V$, Von Neumann's trace inequality gives that the first term verifies

    $$
    \langle \tilde{U}^{\top}D\tilde{V}, \Sigma\rangle \geq -\langle D, \Sigma\rangle
    $$
    This lower bound is attained e.g. for $\tilde{U} = - I_p$ and $\tilde{V}=I_p$, i.e. for $Q = -U$ and $W = V$. Note that any other choice of $\tilde{U}, \tilde{V}$ that minimizes the above inequality leads to the same solution $X^*$.
    We then fall back to the problem of minimizing
    $$
     -\langle D, \Sigma\rangle + \frac{\lambda}4\|D^2 - I_p\|^2
    $$
    with respect to $D$, which is a problem that is separable on the entries of $D$, so that the optimal $D$ is $D = \mathrm{diag}(\sigma_i^*)$ with $\sigma_i^*$ that minimizes the scalar function $x\to -\sigma_i x + \frac\lambda4(x^2 - 1)^2$. We then have $X^* = QDW^{\top} = -U DV^{\top}$.

    The function $x\to -\sigma_i x + \frac\lambda4(x^2 - 1)^2$ is minimized for the value $\sigma_i^*> 1$ such that (cancelling its derivative) it holds 
    $$
    (\sigma_i^*)^3 - \sigma_i^* = \frac{\sigma_i}{\lambda}\enspace.
    $$
    We see that as $\lambda$ goes to $+\infty$ we have $\sigma_i^*$ that goes to $1$ (we recover that $X^*$ tends to orthogonality), and using the implicit function theorem, we find the expansion
    $$
    \sigma_i^* = 1 + \frac{\sigma_i}{2\lambda} +o(\frac1\lambda) \text{ for }\lambda\to+\infty.
    $$
    Furthermore, letting $h(\sigma) = \sigma^3 - \sigma,$ we have 
    $$h(1 + \frac{\sigma_i}{2\lambda}) = \frac{\sigma_i}\lambda + \frac{3\sigma_i^2}{4\lambda^2} +  \frac{\sigma_i^3}{8\lambda^3} \geq \frac{\sigma_i}\lambda.$$
    Since the function $h$ is locally increasing around $1$, we deduce that  $\sigma_i^*  = h^{-1}(\frac{\sigma_i}\lambda) \leq 1 + \frac{\sigma_i}{2\lambda}$.
    Hence, we find $\|X^* - X_{\stiefel}\| = \sqrt{\sum_{i=1}^p(\sigma^*_i - 1)^2} = \frac1{2\lambda}\sqrt{\sum_{i=1}^p\sigma_i^2} + o(\frac{1}{\lambda})$, which goes to $0$ at a $1/\lambda$ rate.

    Finally, the Hessian of $g$ at the optimum is the linear operator $H$ such that 

    $$
    \text{For all } E\in\mathbb{R}^{n\times p}, \enspace H[E] = \lambda\left(E((X^*)^\top 
    X^* - I_p) + X^*(E^{\top}X^* + (X^*)^{\top}E)\right)\enspace.
    $$

    Note that $(X^*)^\top 
    X^* - I_p$ is a symmetric positive-definite matrix. Let us lower-bound the largest eigenvalue of $H$, by taking $E = X^*$. We find 

    \begin{align}
        \langle X^*, H[X^*]\rangle &= \lambda  \left(\langle X^*, X^*((X^*)^\top 
    X^* - I_p)\rangle + 2\langle X^*, X^* (X^*)^{\top}X^*\rangle \right)\\
    &\geq 2\lambda \langle X^*, X^* (X^*)^{\top}X^*\rangle\\
    &\geq 2\lambda (\|X^*\|^2 + \langle X^*, X^* ((X^*)^{\top}X^*-I_p)\rangle)\\ 
    &\geq 2\lambda \|X^*\|^2
    \end{align}
    where for the first and third inequality, we use the positive-definiteness of $(X^*)^\top 
    X^* - I_p$.

    We conclude that $\lambda_{\max}(H) \geq \frac{\langle X^*, H[X^*]\rangle}{\|X^*\|^2}\geq 2\lambda$.
    Similarly, let us upper-bound the smallest eigenvalue of $H$, by taking any $E$ such that $E^{\top}X^* + (X^*)^{\top}E = 0$. We find
     \begin{align}
        \langle E, H[E]\rangle &= \lambda  \langle E, E((X^*)^{\top}X^* - I_p)\rangle\\
        &\leq \lambda \|(X^*)^{\top}X^* - I_p\|_2\|E\|^2\\
        &\leq (\sigma_p + \frac{\sigma_p^2}{4\lambda}) \|E\|^2
    \end{align}
    where we use the inequality $1\leq \sigma_i^* \leq 1 + \frac{\sigma_i}{2\lambda}$ in order to upper bound $\|(X^*)^{\top}X^* - I_p\|_2 = (\sigma_p^*)^2 - 1$ by $\frac{\sigma_p}{\lambda} + \frac{\sigma_p^2}{4\lambda^2}$.
    As a consequence, we have $\lambda_{\min}(H)\leq \frac{\langle E, H[E]\rangle}{\|E\|^2}\leq \sigma_p + \frac{\sigma_p^2}{4\lambda}$.

    Combining the inequalities on the eigenvalues, we find that the conditioning of $H$ is lower-bounded by:
    $$
    \frac{\lambda_{\max}(H)}{\lambda_{\min}(H)}\geq \frac{2\lambda}{\sigma_p + \frac{\sigma_p^2}{4\lambda}} = \frac{2\lambda}{\sigma_p} + o(\lambda) \text{ for }\lambda \to +\infty\enspace.
    $$
\end{proof}
\end{document}